\documentclass[a4paper]{article}
\usepackage[a4paper,top=3cm,bottom=2cm,left=3cm,right=3cm,marginparwidth=1.75cm]{geometry}
\usepackage{microtype}
\usepackage{graphicx}
\usepackage{booktabs} 

\usepackage{hyperref}

\usepackage[numbers]{natbib}

\usepackage{multirow}
\usepackage{amsmath}
\usepackage{float}
\usepackage{amsthm}
\usepackage{amssymb}
\usepackage{thmtools}
\usepackage{thm-restate}
\usepackage{color}
\usepackage{amsfonts}       
\usepackage{nicefrac}       

\usepackage{amsmath,amsfonts,bm}

\def\eqref#1{equation~\ref{#1}}

\def\1{\bm{1}}

\def\mA{{\bm{A}}}

\def\mD{{\bm{D}}}

\DeclareMathAlphabet{\mathsfit}{\encodingdefault}{\sfdefault}{m}{sl}
\SetMathAlphabet{\mathsfit}{bold}{\encodingdefault}{\sfdefault}{bx}{n}

\newcommand{\Var}{\mathrm{Var}}

\newcommand{\bW}{\text{\boldmath{$W$}}}

\newcommand{\bB}{\text{\boldmath{$B$}}}
\newcommand{\bD}{\text{\boldmath{$D$}}}

\newcommand{\bone}{\boldsymbol{1}}

\newcommand{\bI}{\boldsymbol{I}}

\newcommand{\bA}{\boldsymbol{A}}
\newcommand{\bM}{\boldsymbol{M}}
\newcommand{\bbP}{\mathbb{P}}
\newcommand{\bbR}{\mathbb{R}}

\newcommand{\erfc}{\text{erfc}}

\newcommand{\bp}{\mathsf{BP}}

\newtheorem{assumption}{\textbf{Assumption}}

\newtheorem{lemma}{\textbf{Lemma}}
\newtheorem{theorem}{\textbf{Theorem}}

\newcommand{\nn}{\nonumber}

\newcommand{\cN}{\mathcal{N}}

\usepackage{subcaption} 

\usepackage{url}
\graphicspath{{./figs/}}

\begin{document}

	\begin{titlepage}
		\def\thepage{}
		\thispagestyle{empty}
		
		\title{Stabilize Deep ResNet with A Sharp Scaling Factor $\tau$} 
		
		\date{}

		\author{
			Huishuai Zhang\thanks{Microsoft Research Asia, {\tt \{huzhang, tyliu\}@microsoft.com}. Correspond to Huishuai Zhang and Wei Chen.}
			\and
			Da Yu\thanks{Sun Yat-sen University, {\tt yuda3@mail2.sysu.edu.cn}.}
			\and
			Mingyang Yi\thanks{University of Chinese Academy of Sciences, {\tt yimingyang17@mails.ucas.edu.cn}.}
			\and
			Wei Chen\thanks{Institute of Computing Technology, Chinese Academy of Sciences, {\tt chenwei2022@ict.ac.cn}.}
			\and
			Tie-Yan Liu\footnotemark[1] 
		}
		\maketitle
		\begin{abstract}
We study the stability and convergence of training deep ResNets with gradient descent. Specifically, we show that the parametric branch in the residual block should be scaled down by a factor $\tau =O(1/\sqrt{L})$ to guarantee stable forward/backward process, where  $L$ is the number of residual blocks. Moreover, we establish a converse result that the forward process is unbounded when $\tau>L^{-\frac{1}{2}+c}$, for any positive constant $c$. The above two results together establish a sharp value of the scaling factor in determining the stability of deep ResNet.  Based on the stability result, we further show that gradient descent finds the global minima if the ResNet is properly over-parameterized, which significantly improves over the previous work with a much larger range of $\tau$  that admits  global convergence. Moreover, we show that the convergence rate is independent of the depth,  theoretically justifying the advantage of ResNet over vanilla feedforward network. Empirically, with such a factor $\tau$, one can train deep ResNet without normalization layer. Moreover for ResNets with normalization layer, adding  such a factor $\tau$ also stabilizes the training and obtains significant performance gain for deep ResNet. 
\end{abstract}
	\end{titlepage}

\section{Introduction}\label{sec:introduction}

Residual Network (ResNet) has achieved great success in computer vision tasks  since the seminal paper \citep{he2016deep}. The ResNet structure has also been extended to natural language processing and achieved the state-of-the-art performance  \citep{vaswani2017attention, devlin2018bert}. In this paper, we study the forward/backward stability and convergence of training deep  ResNet with gradient descent. 

Specifically, we consider the following residual block \citep{he2016deep}, 
\begin{flalign}
&\text{ residual block: } \quad h_{l}=\phi(h_{l-1}+\tau \mathcal{F}_l(h_{l-1})),   \label{eq:resnet-block}
\end{flalign}
where $\phi(\cdot)$ is the point-wise activation function, $h_l$ and $h_{l-1}$ are the output and input of the residual block $l$, $\mathcal{F}_l(\cdot)$ is the parametric branch, e.g., $\mathcal{F}_l(h_{l-1}) = \bW_{l}h_{l-1}$ and $\bW_{l}$ is the trainable parameter, and $\tau$ is a scaling factor on the parametric branch. 

We note that standard initialization schemes, e.g., the Kaiming initialization or the Glorot initialization, are designed to keep the forward and backward variance constant when passing through one layer. However, things become different for ResNet because of the existence of the identity mapping. If $\bW_l$ adopts the standard initialization, a small $\tau$ is necessary for a stable forward process of deep ResNet, because the output magnitude quickly explodes for $\tau=1$ as $L$ gets large. On the other side, a limit form of \emph{Euler's constant} indicates that $\tau=O(1/L)$ is sufficient for the forward/backward stability, which is assumed in previous work \citep{allen2018convergence,du2019gradient}. We ask 

``Are there other values of $\tau$ that can guarantee the stability of ResNet with arbitrary depth?''

We answer the above question affirmatively by establishing a non-asymptotic analysis that the stability is guaranteed for deep ResNet with arbitrary depth as long as $\tau =O(1/\sqrt{L})$. Moreover conversely,  for any positive constant $c$, if $\tau = L^{-\frac{1}{2}+c}$, the network output norm grows at least with rate $L^c$ in expectation, which implies the forward/backward process is unbounded as $L$ gets large.

One step further, based on the stability result, we show that if the network is properly over-parameterized, gradient descent  finds global minima for training ResNet with $\tau\le \tilde{O}(1/\sqrt{L})$  \footnote{We use $\tilde{O}(\cdot)$ to hide logarithmic factors.}. This is essentially different from previous work that assumes $\tau \le \tilde{O}(1/L)$ \cite{allen2018convergence,du2018gradient,frei2019algorithm}. 

Our contribution is summarized as follows. 
\begin{itemize}
\item We establish a non-asymptotic analysis showing that $\tau =1/\sqrt{L}$ is sharp in the order sense to guarantee the stability of deep ResNet.

\item For $\tau\le \tilde{O}(1/\sqrt{L})$, we establish the convergence of gradient descent to global minima for training over-parameterized ResNet with a depth-independent rate. 

\end{itemize}

The key step to prove our first claim is a new  bound of the spectral norm  of the forward process for ResNet with $\tau= O(1/\sqrt{L})$. We find that, although the natural bound $(1+1/\sqrt{L})^L$ explodes, the randomness of the trainable parameter in the parametric branch helps to control the output norm growth. Specifically, we bound the mean and the variance  about the largest possible change after deep residual mappings  when $\tau=O(1/\sqrt{L})$. 

We also argue the advantage of adding $\tau$ over other stabilization methods, such as \emph{batch normalization} (BN) \cite{ioffe2015batch} and \emph{Fixup} \cite{zhang2019fixup}. First,  it has advantage over BN to guarantee stability. BN is architecture-agnostic and the output norm of ResNet with BN still grows unbounded as the depth increases. In practice, it has to to employ a learning rate warm-up stage to train very deep ResNet even with BN \cite{he2016deep}. In comparison, we prove that ResNet with $\tau$ is stable over all depths and hence does not require any learning rate warm-up stage. Second, it is also more stable than the approach of scaling down initialization that is adopted in \emph{Fixup}. Scaling down initial residual weight does not scale down the gradient properly and   \emph{Fixup} could explode after gradient descent updates for deep ResNet.

At last, we corroborate the theoretical findings with extensive experiments. First, we demonstrate that with $\tau=1/\sqrt{L}$, ResNet can be effectively trained without the normalization layers. It is more stable and achieves better performance than \emph{Fixup}. Second, we demonstrate that adding $\tau=1/\sqrt{L}$ on top of the normalization layer can obtain even better performance.

\subsection{Related works}
There is a large volume of literature studying ResNet. We can only give a partial list. 

To argue the benefit of skip connection, \citet{veit2016residual} interpret ResNet as an ensemble of shallower networks, \citet{zhang2018local} study the local Hessian of residual blocks, \citet{hardt2016identity} show that deep linear residual networks have no spurious local optima,  \citet{orhan2018skip} observe that skip connection eliminates the singularity, and \citet{balduzzi2017shattered} find that ResNet is  more resistant to the gradient shattering problem than the feedforward network. However, these results mainly rely on empirical observation or strong model assumption. 

There are also several papers studying ResNet from the  stability perspective \cite{arpit2019initialize, zhang2019fixup, zhang2019towards, yang2017mean, haber2017stable}. In comparison, we study the model closest to the original ResNet and provide a rigorous non-asymptotic analysis for the stability when $\tau=O(1/\sqrt{L})$ and  a converse result showing the sharpness of $\tau$. We also demonstrate the empirical advantage of learning ResNet with $\tau$.

Our work is also related to recent literature on the theory of learning deep neural network with gradient descent in the over-parameterized regime. Many works \cite{jacot2018neural,allen2018convergence, du2018gradient, chizat2018global,  zou2018stochastic, zou2019improved,  arora2019exact, oymak2018overparameterized,chen2019much, ji2019polylogarithmic} use Neural Tangent Kernel (NTK) or similar technique to argue the global convergence of gradient descent for training over-parameterized deep neural network.  Some \cite{brutzkus2017sgd, li2018learning, allen2019learning, arora2019fine, cao2019generalization, neyshabur2018role} study the generalization properties of over-parameterized neural network. On the other side, there are papers \cite{ghorbani2019limitations,chizat2019lazy, yehudai2019power, allen2019can} discussing the limitation of the NTK approach in characterizing the behavior of neural network. Additionally, several papers  \cite{chizat2018note,mei2018mean,mei2019mean, nguyen2019mean, fang2019over,fang2019convex} study the convergence of the weight distribution in the probabilistic space via gradient flow for two or multiple layers network. To the best of our knowledge, we are  the first to provide the global convergence of learning ResNet in the regime of  $\tau\le O(1/\sqrt{L})$ 

\section{Preliminaries \label{sec:model}}
There are many residual network models since the seminal paper \citep{he2016deep}. Here we  study a simplified ResNet that shares the same structure as \cite{he2016deep} \footnote{In \cite{he2016deep}, there is a ReLU after the building block $y=x+F(x)$ (please refer to Figure. 2 in \cite{he2016deep}), and hence a whole residual block is $h_l = \phi(h_{l-1} + F(h_{l-1}))$ (if using the notations in our paper). }, which is described as follows,
\begin{itemize}
\item Input layer: $h_{0}=\phi(\bA x)$, where $x\in\bbR^{p}$ and $\bA\in\bbR^{m\times p}$;
\item $L-1$ residual blocks: $h_{l}=\phi(h_{l-1}+\tau\bW_{l}h_{l-1})$, where $\bW_{l}\in\bbR^{m\times m}$; 
\item A fully-connected layer: $h_{L}=\phi(\bW_{L}h_{L-1})$, where $\bW_{L}\in\bbR^{m\times m}$;
\item Output layer: $y=\bB h_{L}$, where $\bB\in\bbR^{d\times m}$;
 \item Initialization: The entries of $\bA, \bW_l$ for $l\in [L-1]$, $\bW_L$ and $\bB$ are independently sampled from $\cN(0,\frac{2}{m})$, $\cN(0,\frac{1}{m})$, $\cN(0,\frac{2}{m})$ and $\cN(0,\frac{1}{d})$, respectively;
\end{itemize}
where $\phi(\cdot):=\max \{0, \cdot\}$ is the ReLU activation function. We assume the input dimension is $p$, the intermediate layers have the same width $m$ and the output has dimension $d$. For a positive integer $L$, we use $[L]$ to denote the set $\{1, 2, ..., L\}$.  We denote the values before activation by $g_{0}=\bA x,g_{l}=h_{l-1}+\tau\bW_{l}h_{l-1}$ for $l=[L-1]$ and $g_{L}=\bW_{L}h_{L-1}$. We use $h_{i,l}$ and $g_{i,l}$ to denote the value of $h_{l}$ and $g_{l}$, respectively, when the input vector is $x_{i}$, and $\bD_{i,l}$ the diagonal activation matrix where $[\bD_{i,l}]_{k,k}=\bone_{\{(g_{i,l})_{k}\ge0\}}$.  We use superscript $^{(0)}$ to denote the value at initialization, e.g., $\bW_l^{(0)}$, $h_{i,l}^{(0)}$, $g_{i,l}^{(0)}$ and $\bD_{i,l}^{(0)}$. We may omit the subscript $_i$ and the superscript $^{(0)}$ when they are clear from the context for simplifying the notations.

We introduce a notation $\overrightarrow{\bW}:=(\bW_1, \bW_2, \dots, \bW_L)$ to represent all the trainable parameters. We note that $\bA$ and $\bB$ are fixed after initialization. Throughout the paper, we use $\|\cdot\|$ to denote the $l_2$ norm of  a vector. We further use $\|\cdot\|$ and $\|\cdot\|_F$  to denote the spectral norm and the Frobenius norm of a matrix, respectively. Denote $\|\overrightarrow{\bW}\| := \max_{l\in[L]}\|\bW_l\|$ and $\|\bW_{[L-1]}\| := \max_{l\in[L-1]}\|\bW_l\|$.

The training data set is $\{(x_i, y_i^*)\}_{i=1}^n$, where $x_i$ is the feature vector and $y_i^*$ is the target signal for $i=1, ..., n$.  We consider  the objective function is $F(\overrightarrow{\bW}):=\sum_{i=1}^{n}F_{i}(\overrightarrow{\bW})$ where $F_{i}(\overrightarrow{\bW}):=\ell(\bB h_{i,L}, y_{i}^{*})$ and $\ell(\cdot)$ is the loss function. The model is trained by running the gradient descent algorithm. Though ReLU is nonsmooth, we abuse the word ``gradient" to represent the value computed through back-propagation. 

\section{Forward and Backward Stability of ResNet}\label{sec:stability}

In this section, we establish the stability of training ResNet. We show that when  $\tau = O(1/\sqrt{L})$ the forward and backward pass is bounded at  the initialization and after small perturbation. On the converse side, for an arbitrary positive constant $c$, if $\tau>L^{-0.5+c}$, the output magnitude grows at least polynomial with depth at the initialization. We also argue the advantage of using a factor $\tau$ over other stabilization methods, such as BN and \emph{Fixup}. The stability result forms the basis to establish the global convergence in Section~\ref{sec:main-result}. 

\subsection{Forward process is bounded if $\tau= O(1/\sqrt{L})$}

We first give a  non-asymptotic bound on the forward process at initialization.

\begin{restatable}{theorem}{spectralnormtight}\label{thm:initial-spectral-norm}

Suppose that $\overrightarrow{\bW}^{(0)}$, $\bA$ are randomly generated as in the initialization step, and  $\bD^{(0)}_{i, 0},\dots,\bD^{(0)}_{i, L}$  are  diagonal activation matrices for $i\in [n]$. Suppose that $c$ and $\epsilon$ are arbitrary positive constants with $0<\epsilon <1$. If $\tau$ satisfies $\tau^2L \le \min \{ \frac{1}{2}\log(1+c), \frac{\log^2(1+c)}{16(1+\log(1+2/\epsilon))}\}$,  then with probability at least $1-3nL^2\cdot\exp\left(-m\right)$ over the initialization randomness, we have for any two integers $a,b\in [L-1]$ with $b>a$ and for all $i\in [n]$, {\small
\begin{flalign}
\left\|\bD^{(0)}_{i,b}\left(\bI+\tau\bW_{b}^{(0)}\right)\cdots \bD^{(0)}_{i,a}\left(\bI+\tau\bW_{a}^{(0)}\right)\right\|\le \frac{1+c}{1-\epsilon}. \label{eq:resnetupperbound}
\end{flalign}}
\end{restatable}

{The proof is based on Markov's inequality with recursively conditioning. The full proof is deferred to  Appendix \ref{app:thm:initial-spectral-norm}. Here we give an outline.

\begin{proof}[Proof Outline] 
We omit the subscript $i$ and the superscript $(0)$ for simplicity. Suppose that $\|h_{a-1}\|=1$.  Let $g_l = h_{l-1}+\tau \bW_lh_{l-1}$ and $h_l = \bD_l g_{l}$ for $l=\{a,..., b\}$. We have 
\begin{flalign*}
\|h_{b}\|^2& = \frac{\|h_{b}\|^2}{\|h_{b-1}\|^2}  \cdots\frac{\|h_{a}\|^2}{\|h_{a-1}\|^2}\|h_{a-1}\|^2
\le \frac{\|g_{b}\|^2}{\|h_{b-1}\|^2}  \cdots\frac{\|g_{a}\|^2}{\|h_{a-1}\|^2}\|h_{a-1}\|^2,
\end{flalign*}
where the inequality is due to that $\|\bD_l\|\le 1$. Taking logarithm at both side, we have
\begin{equation*}
\log{\|h_{b}\|^2}\le \sum_{l=a}^{b}\log \Delta_{l},\quad  \text{where  } \Delta_{l} := \frac{\|g_{l}\|^2}{\|h_{l-1}\|^2}. 
\end{equation*}
If let $\tilde{h}_{l-1} := \frac{h_{l-1}}{\|h_{l-1}\|}$, then we obtain that{\small
\begin{equation*}
\begin{aligned}
\log{\Delta_{l}} &= \log\left(1 + 2\tau\left\langle \tilde{h}_{l-1},\bW_{l}\tilde{h}_{l-1} \right\rangle + \tau^{2}\|\bW_{l}\tilde{h}_{l-1}\|^{2}\right)\\
&\leq  2\tau\left\langle \tilde{h}_{l-1},\bW_{l}\tilde{h}_{l-1} \right\rangle + \tau^{2}\|\bW_{l}\tilde{h}_{l-1}\|^{2},
\end{aligned}
\end{equation*}}
where the inequality is because $\log (1+x) < x$ for $x>-1$. Let $\xi_{l} := 2\tau\left\langle \tilde{h}_{l-1},\bW_{l}\tilde{h}_{l-1} \right\rangle$ and $\zeta_{l}:= \tau^{2}\|\bW_{l}\tilde{h}_{l-1}\|^{2}$. Then given $\tilde{h}_{l-1}$, we have $\xi_{l}\sim \cN\left(0, \frac{4\tau^2}{m}\right)$, $\zeta_{l}\sim \frac{\tau^{2}}{m}\chi_{m}^2$. 

We can argue that $\sum_{l=a}^{b} \xi_l \sim \cN\left(0, \frac{4(b-a)\tau^2}{m}\right)$ and $\sum_{l=a}^{b} \zeta_l \sim  \frac{(b-a)\tau^{2}}{m}\chi_{m}^2$. Hence for arbitrary positive constant $c_1$, if $\tau^2L \le c_1/4$ then $ \sum_{l=a}^{b}\log \Delta_{l}\le c_1$ with probability at least $1- 3\exp(-\frac{mc_1^2}{64\tau^2L})$. We then convert the condition on $c_1$ to that on $c$ in the theorem. Taking $\varepsilon$-net argument, we can establish the spectral norm bound for all vector $h_{a-1}$. Let $a$ and $b$ vary from $1$ to $L-1$ and taking the union bound gives the claim. The full proof is presented in Appendix \ref{app:thm:initial-spectral-norm}.
\end{proof}}

We note that the constant $c$ and $\epsilon$ can be chosen arbitrarily small such that $(1+c)/(1-\epsilon)$ is arbitrarily close to $1$ given stronger assumption on $\tau^2L$. Theorem 1 indicates that the norm of every residual block output is upper bounded by $(1+c)/(1-\epsilon)$ if the input vector has norm 1, which demonstrates that the the forward process is stable. 
This result is a bit surprising since for $\tau = O(1/\sqrt{L})$  a natural  bound on the spectral norm $\|(\bI+\tau\bW^{(0)}_L)\cdots (\bI+\tau\bW^{(0)}_1)\|\le (1+\frac{1}{\sqrt{L}})^L$ explodes. Here the intuition is that the cross-product term concentrates on the mean 0 because of the independent randomness of matrices $\bW_l^{(0)}$ and the variance can be bounded at the same time.

Moreover, we can also establish a lower bound on the output norm of each residual block as follows. 
\begin{restatable}{theorem}{hnorminitialization}\label{thm:hnorm-initialization}
Suppose that $c$ is an arbitrary constant with $0<c<1$. If $\tau^2L \le  \frac{1}{4}\log (1-c)^{-1}$, then with probability at  least $1-2nL^2\cdot \exp\left(-\frac{1}{32}m\log(1-c)^{-1}\right)$
over the randomness of $\bA\in\bbR^{m\times p}$ and $\overrightarrow{\bW}^{(0)}\in(\bbR^{m\times m})^{L}$ the following holds
{\small
\begin{flalign}
\forall i\in[n],l\in[L-1]:\;\left\|h_{i,l}^{(0)}\right\|\ge 1-c.
\end{flalign}}
\end{restatable}

\begin{proof}
The proof is similar to that of Theorem 1 but harder. The high level idea is to control the mean and the variance of the mapping of the intermediate residual blocks  simultaneously by utilizing the Markov's inequality with the recursive conditioning.  The full proof is deferred to Appendix \ref{app:thm:hnorm-initialization}.
\end{proof}

Combining these two theorems, we can conclude that the  norm of each residual block output concentrates around $1$ with high probability $1-O(nL^2) \exp(-\Omega(m))$. Moreover these two theorems also holds for $\overrightarrow{\bW}$ that is within the neighborhood of $\overrightarrow{\bW}^{(0)}$, which is presented in Appendix \ref{app:sec:perturbed-stability}.

\subsection{Backward process is bounded for $\tau\le O(1/\sqrt{L})$} \label{subsec:backward-stability}

For ResNet, the gradient with respect to the parameter is computed through back-propagation. For any input sample $i$, we denote $\partial h_{i,l} := \frac{\partial F_i(\overrightarrow{\bW})}{\partial h_{i,l}}$ and $ \nabla_{\bW_{l}}F_i(\overrightarrow{\bW}) := \frac{\partial F_i(\overrightarrow{\bW})}{\partial \bW_{l}} =\left(\tau \mD_{i,l} \partial h_{i,l}\right)\cdot h_{i,l-1}^T$. Therefore, the gradient upper bound is guaranteed if  $h_{i,l}$ and  $\partial h_{i,l}$ are bounded for all blocks. We next show that the backward process is  bounded for each individual sample at the initialization stage. 

\begin{restatable}{theorem}{gradientupperbound}\label{thm:gradient-upperbound}
For every input sample $i\in[n]$ and for any positive constants $c$ and $\epsilon$ with $0<\epsilon<1$, if $\tau$ satisfies $\tau^2L \le \min \{ \frac{1}{2}\log(1+c), \frac{\log^2(1+c)}{16(1+\log(1+2/\epsilon))}\}$,  then with probability at least $1- 3nL^2\cdot \exp\left(-\frac{1}{4}mc^2\right)$ over the randomness of $\bA,\bB$ and $\overrightarrow{\bW}^{(0)}$, the following holds $\forall l\in[L-1]$  
{\small
\begin{align}\label{eq:gradient-upperbound}
&\|\nabla_{\bW_{l}}F_i(\overrightarrow{\bW}^{(0)})\|_{F}\le \frac{(1+c)^2}{(1-\epsilon)^2}(2\sqrt{2}+c)\tau \|\partial h_{i,L}\|, \;\; \|\nabla_{\bW_{L}}F_i(\overrightarrow{\bW}^{(0)})\|_{F}\le \frac{(1+c)}{(1-\epsilon)}\|\partial h_{i,L}\|. 
\end{align}}
\end{restatable}

{The full proof is is deferred to Appendix \ref{app:thm:gradient-upperbound}. Here we give an outline.

\begin{proof}[Proof Outline]
The argument is based on the back-propagation formula and Theorem \ref{thm:initial-spectral-norm}. We omit the superscript $^{(0)}$ for notation simplicity. For each $i\in[n]$ and $l\in[L-1]$, i.e., the residual layers, we have
{\small
\begin{flalign*}
\|\nabla_{\bW_{l}}F_i(\overrightarrow{\bW})\|_{F} &=\left\|\tau \left(\bD_{i,l} (\bI+\tau \bW_{l+1})^T\cdots \bD_{i,L-1}\bW_L^T\bD_{i,L} \partial h_{i,L}\right) h_{i, l-1}^T\right\|_F \\
&\le  \tau\|\bD_{i,l} (\bI+\tau \bW_{l+1})^T\cdots \bD_{i,L-1}\| \cdot \|\bW_L^T\bD_{i,L}\| \cdot \|\partial h_{i,L}\|\cdot \|h_{i,l-1}\|,\\
&\le \frac{(1+c)^2}{(1-\epsilon)^2}(2\sqrt{2}+c)\tau \|\partial h_{i,L}\|,
\end{flalign*}}
where the last inequality is due to Theorem \ref{thm:initial-spectral-norm} and the  spectral norm bound of $\bW_L$  given in Appendix \ref{app:useful-lemmas}. The full proof is deferred to Appendix \ref{app:thm:gradient-upperbound}.
\end{proof}}

This theorem indicates that the gradient of residual layers could be $\tau$ times smaller than the usual feedforward layer. Moreover, for ResNet with $\tau=1/\sqrt{L}$, the norm of all layer gradient is independent of the depth, which allows us to use a \emph{depth independent learning rate} to train ResNets of all depths. This is essentially different from the feedforward case \cite{allen2018convergence,zou2019improved}.  We note that the gradient upper bound also holds for $\overrightarrow{\bW}$ within the neighborhood of $\overrightarrow{\bW}^{(0)}$ (see details in Appendix \ref{app:sec:perturbed-stability} and \ref{app:thm:gradient-upperbound}).

\subsection{A converse result for $\tau > L^{-\frac{1}{2}+c}$ \label{subsec:converse}}

We have built the stability of the forward/backward process for $\tau = O(1/\sqrt{L})$. We next establish a converse result showing that if $\tau$ is slightly larger than $L^{-\frac{1}{2}}$ in the order sense, the network output norm grows uncontrollably as the depth $L$ increases. This justifies the sharpness of the value $\tau = 1/\sqrt{L}$.  Without loss of generality, we assume $\|h_0\|=1$. 

\begin{restatable}{theorem}{conversetau}\label{thm:converse_tau}
Suppose that $c$ is an arbitrary positive constant and the ResNet is defined and initialized as in Section \ref{sec:model}. If  $\tau\ge L^{-\frac{1}{2}+c}$, then we have
\begin{flalign}
 \mathbb{E}\|h_{L}\|^2 \ge \frac{1}{2}L^{2c}.
\end{flalign}
\end{restatable}

\begin{proof}
The proof is based on a new inequality $(h_{l})_{k}\geq\phi\left(\sum_{a=1}^{l}\left(\tau\bW_{a}h_{a-1}\right)_{k}\right)$ for $l\in[L-1]$ and for $k\in[m]$. By the symmetry of Gaussian variables and the recursive conditioning, we can compute the expectation of $\|h_L\|^2$ exactly. The whole proof is relegated to Appendix \ref{app:sec:proof-converse-tau}.
\end{proof}

This indicates that $\tau= O(1/\sqrt{L})$ is sharp to guarantee the forward stability of deep ResNet. We note that Theorem \ref{thm:initial-spectral-norm} and \ref{thm:gradient-upperbound} hold with high probability when $m> \Omega(\log L)$ and Theorem \ref{thm:hnorm-initialization} holds with high probability when $m>\Omega(\log(nL))$. These are very mild conditions on the width $m$, which are satisfied by practical networks.

\subsection{Comparison with other approaches for stability}\label{subsec:comparison}

Up to now, we have provided a sharp value of $\tau$ in terms of determining the stability of deep ResNet. In practice, two other approaches are used in residual networks to provide the stability: adding normalization layers, e.g., batch normalization (BN) \citep{ioffe2015batch}, and scaling down the initial residual weights, e.g., \emph{Fixup} \citep{zhang2019fixup}. Next, we discuss BN and \emph{Fixup} from the stability perspective, respectively, and make comparison with adding $\tau=1/\sqrt{L}$.

Batch normalization is placed right after each convolutional layer in \citep{he2016deep}. Here for the ResNet model defined in Section~\ref{sec:model}, we put BN after each parametric branch and the residual block becomes $h_l = \phi(h_{l-1} + \tilde{z}_l)$, where $(\tilde{z}_{i,l})_k := \text{BN}\left((z_{i,l})_k\right) = \frac{(z_{i,l})_k- \mathbb{E}[(z_{\cdot, l})_k]}{\sqrt{\Var[(z_{\cdot, l})_k]}}$ and $(z_{i,l})_k := \left(\bW_lh_{i,l-1}\right)_k$ for $k=[m]$ and $l=[L-1]$, and the expectation and the variance are taken over samples in a mini-batch. Then we have $\mathbb{E} (\tilde{z}_{\cdot, l})_k =0$ and $\Var[(\tilde{z}_{\cdot,l})_k] = 1$. We use the following proposition to estimate the norm of  each residual block output for the ResNet with BN.

\begin{restatable}{proposition}{normwithbn}\label{clm:norm-with-bn}
Assume that $(\tilde{z}_{l})_k$ are independent random variable over $l,k$ with $\mathbb{E} (\tilde{z}_{l})_k =0$ and $\Var[(\tilde{z}_{l})_k] = 1$. The    output norm of the residual block $l$  satisfies $\mathbb{E}\|h_{l}\|^2 \ge \frac{1}{2}ml$, for $l\in [L-1]$.
\end{restatable}
\begin{proof}
The proof is adapted from the proof of Theorem \ref{thm:converse_tau}, and is presented in Appendix \ref{app:sec:proof-converse-tau}.
\end{proof}

This indicates that the block output norm of ResNet with BN grows roughly at the rate $\sqrt{l}$ at the initialization stage, where $l$ is the block index and the larger $l$ the closer to the output. To verify this, we plot how the output norm of each residual block grows for ResNet1202 (with/without BN)\footnote{Throughout the paper, the naming rule of ResNet is as follows.``ResNet" is referred to the model defined in Section \ref{sec:model}, ``ResNet\#" is referred to the models in \citet{he2016deep} with removing all the BN layers, e.g., ResNet1202, ``ResNet\#$+$BN" corresponds to the original model in \cite{he2016deep}, ``$+$Fixup" corresponds to initializing the model with Fixup, and ``$+\tau$" is referred to adding $\tau$ on the output of the parametric branch in each residual block.}  in Figure \ref{fig:normgrow-bn}. We see that at epoch $0$ (initialization stage), the output norm grows almost with the rate $\sqrt{l}$ as predicted in Proposition \ref{clm:norm-with-bn}. After training, the estimation in Proposition \ref{clm:norm-with-bn} is not as accurate as the initialization because the independence assumption does not hold after training. Besides the output norm growth, in practice, \citet{he2016deep} have to use warm-up learning rates to train very deep ResNets, e.g., ResNet1202$+$BN. In contrast, it is proved that the approach of adding $\tau=1/\sqrt{L}$ is stable over all depths and hence does not require any learning rate warm-up stage.

\begin{figure}
\centering
    \begin{minipage}{0.6\textwidth}
        \begin{center}
        \centerline{\includegraphics[width=1.0\linewidth]{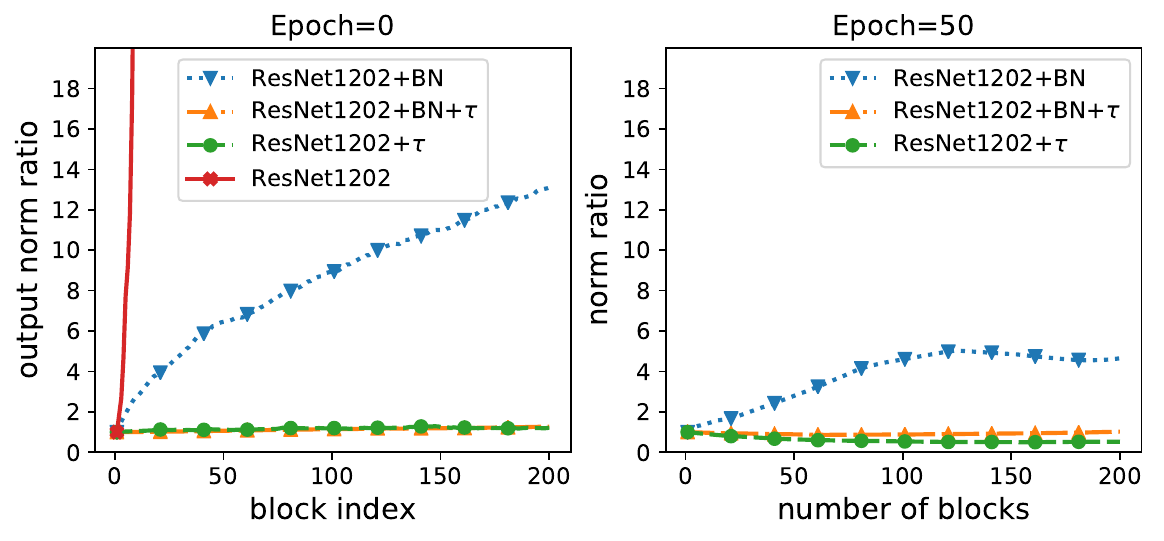}}
        \caption{The $l_2$ norm of residual block output of the first stage of ResNet1202 at epoch 0 and epoch 50.  The X axis is the block index and the Y axis is the output norm ratio compared to the first block.} 
        \label{fig:normgrow-bn}
        \end{center}
    \end{minipage}
    \hfill
    \begin{minipage}{0.3\textwidth}
        \begin{center}
        \centerline{\includegraphics[width=1.0\linewidth]{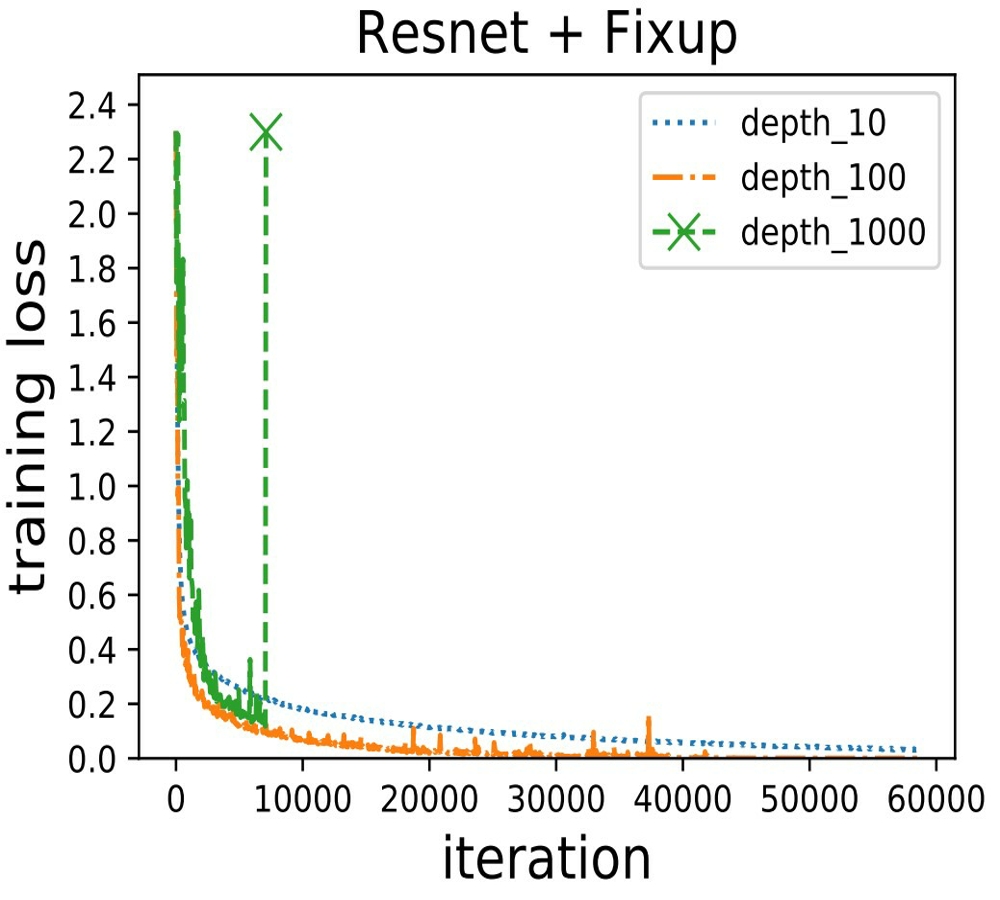}}
        \caption{Training curves of ResNets with Fixup: MNIST classification, width $m=128$ and learning rate $\eta=0.01$.}
        \label{fig:fixup-mnist}
        \end{center}
    \end{minipage}
\vspace{-2mm}
\end{figure}

Recently, \citet{zhang2019fixup} propose \emph{Fixup} to train residual networks without the normalization layer. Essentially for each residual block, \emph{Fixup} sets the weight matrix near the output to be 0  at the initialization stage, and then scales down the all other weight matrices by a factor that is determined by the network structure.  However, in practice \emph{Fixup} does not always converge for training very deep residual networks as shown in Section \ref{subsec:exp:learntau}. Moreover, for the ResNet model defined in Section~\ref{sec:model}, \emph{Fixup} could be unstable after gradient updates. The residual block is given by $h_l = \phi(h_{l-1} + \bW_l h_{l-1})$, and following \emph{Fixup}, $\bW_l^{(0)}$ is initialized to be $0$ for $l \in[L-1]$. At the initial stage for input sample $i$, $h_{i,l} = h_{i,0}$ and hence $\nabla_{\bW_{l}}F_i = \partial h_{i,L-1}\cdot h_{i,0}^T$, the same for all $l \in[L-1]$.   Then after one gradient update  the residual blocks mapping $\prod_{l=1}^{L-1}\bD_{i,l}(\bI +\eta \cdot\nabla_{\bW_{l}}F_i )$ could behave like $(\bD(\bI +\eta \cdot \partial h_{i,L-1}\cdot h_{i,0}^T))^{L-1}$ when $\bD_{i,l}=\bD$ for all $l$, which grows exponentially.  Empirically, such explosion is observed for deep ResNet with \emph{Fixup} (see Figure~\ref{fig:fixup-mnist}).  In contrast, the ResNet with $\tau$ is stable for varying depths (see Figure~\ref{fig:width}), as guaranteed by our theory. 

\section{Global Convergence for Over-parameterized ResNet}\label{sec:main-result}

In this section, we establish that gradient descent converges to global minima for learning an over-parameterized ResNet with $\tau \le \tilde{O}(1/\sqrt{L})$. Compared to the recent work \citep{allen2018convergence}, our result significantly enlarges the region of $\tau$ that admits the global convergence of gradient descent. Moreover, our result also theoretically justifies the advantage of ResNet over vanilla feedforward network in terms of facilitating the convergence of gradient descent. Before stating the theorem, we introduce common assumptions on the training data and the loss function \citep{allen2018convergence, zou2019improved, oymak2018overparameterized}. 

\begin{assumption}[training data] \label{assum:data}
For any $x_i$, it holds that $\|x_i\|=1$ and $(x_i)_p = 1/\sqrt{2}$.  There exists $\delta>0$, such that $\forall i,j \in[n], i\neq j, \|x_{i}-x_{j}\|\ge\delta$. 
\end{assumption}

The loss function $\ell(\cdot, \cdot)$ is quadratic and the individual objective is $F_{i}(\overrightarrow{\bW}):=\frac{1}{2}\|\bB h_{i,L}-y_{i}^{*}\|^{2}$. {We note that the assumption $(x_i)_p=1/\sqrt{2}$ means that the last coordinate of every $x_i$ is $1/\sqrt{2}$. This gives a random bias term after the first layer $\mA(\cdot)$, which makes the proof of Lemma 6 for the gradient lower bound easier. This assumption is because of the proof convenience rather than something that should be satisfied in practice.}

\begin{restatable}{theorem}{mainresult} \label{thm:main-result}
Suppose that the ResNet is defined and initialized as in Section \ref{sec:model} with  $\tau \le  O(1/(\sqrt{L}\log m))$ and the training data satisfy Assumption \ref{assum:data}. If the network width $m\ge \Omega(n^{8} L^7\delta^{-4}d\log^2 m)$,
then with probability at least $1-\exp(-\Omega(\log^{2}m))$, gradient
descent with learning rate $\eta=\Theta(\frac{d}{nm})$
finds a point $F(\overrightarrow{\bW})\le\varepsilon$ in $T=\Omega(n^2\delta^{-1}\log \frac{n\log^2 m}{\varepsilon})$
iterations.
\end{restatable}
\begin{proof}
The full proof is deferred to  Appendix \ref{app:thm:main-result}. 
\end{proof}

This theorem establishes the linear convergence of gradient descent for learning ResNet for the range  $\tau\le O(1/(\sqrt{L}\log m))$.  Combined with the unstable case of  $\tau> 1/\sqrt{L}$ in Section \ref{subsec:converse}, we give a nearly full characterization of the convergence in terms of the range of $\tau$. Moreover, our result indicates that the learning rate and the total number of iterations are depth-independent. We note that a recent paper \cite{frei2019algorithm} also achieves a depth-independent rate but only for the case $\tau \le O(1/(L\log m))$, whose proof critically relies on the choice of $\tau=1/L$. The overparameterization dependence and the number of iterations are not directly comparable as we are studying the regression problem while \cite{frei2019algorithm} is for the classification problem with different data assumption. Other previous results \cite{allen2018convergence, du2018gradient} characterize the convergence guarantee only for the case $\tau \le O(1/(L\log m))$, and their total number of iterations scales with the order $L^2$. Our depth-independent results are achieved by a tighter smoothness and gradient upper bound.

In the analysis with the feedforward case \citep{allen2018convergence, zou2019improved}, the learning rate has to  scale with $1/L^2$ and the total number of iterations scales with $L^2$ for the convergence of learning feedforward network. Therefore, our result theoretically justifies the advantage of ResNet over vanilla feedforward network in terms of facilitating the convergence of gradient descent. 

{Finally, we add a remark on the width requirement in Theorem \ref{thm:main-result}. The width grows polynomially with the number of training examples. Such dependence is because we need to more neurons to distinguish each data point sufficiently with more examples, which is common for the regression task \cite{allen2019convergence, zou2019improved}. This dependence could be avoid by assuming the training data follows specific distributions for the classification task \cite{cao2020generalization}. However this is orthogonal to our main claim that ResNet converges with a \emph{depth-dependent} rate.}

 \begin{figure}
\begin{center}
\centerline{\includegraphics[width=1.0\columnwidth]{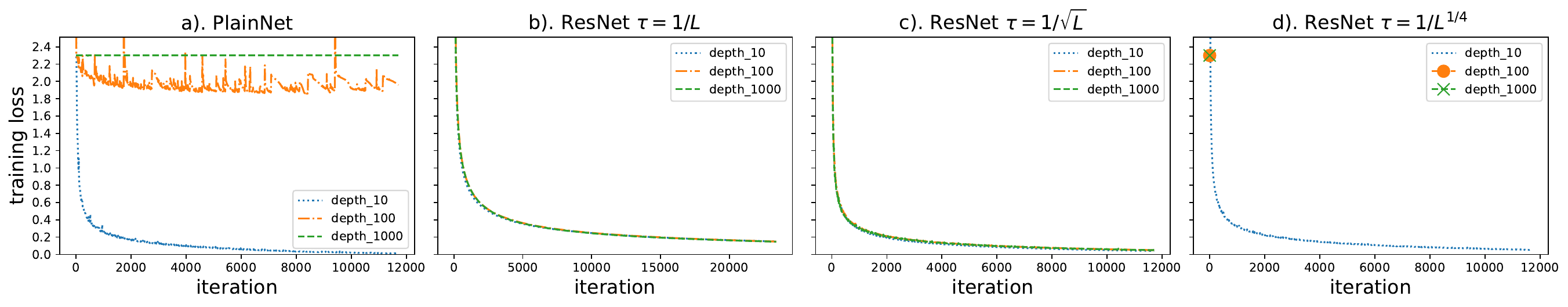}}
\caption{Training curves for PlainNet, ResNet with $\tau=\frac{1}{L}$, $\tau =\frac{1}{\sqrt{L}}$ and $\tau=\frac{1}{L^{1/4}}$ (from left to right). We use markers to denote the training encounters numerical overflow.}
\label{fig:width}
\end{center}
\end{figure}

\section{Empirical Study}\label{sec:experiment}

In this section, we present experiments to verify our theory and show the practical value of ResNet with $\tau$. We first compare the performance of ResNet with different $\tau$'s and demonstrate that $\tau=\frac{1}{\sqrt{L}}$ is a sharp value in determining the trainability of deep ResNet. We then compare the performance of adding the factor $\tau$ and using \emph{Fixup} initialization when training the popular residual networks without normalization layers. We finally show that with normalization layer, adding $\tau$ also significantly improve the performance for both CIFAR and ImageNet tasks. {Source code available online \hyperlink{https://github.com/dayu11/tau-ResNet}{https://github.com/dayu11/tau-ResNet}.}

\subsection{Theoretical verification}\label{subsec:experiment-theory}

We train feedforward fully-connected neural networks (PlainNet), ResNets with different values of $\tau$, and compare their convergence behaviors.  {Specifically, for ResNets, we adopts the exactly the same residual architecture as described in Eq.~(\ref{eq:resnet-block}) and Section \ref{sec:model}.}
The PlainNet adopts the same architecture as the ResNets without the skip connection.  The models are generated with width $m=128$ and  depth $L\in \{10, 100, 1000\}$. For ResNets with $\tau$, we choose $\tau =\frac{1}{L}, \frac{1}{\sqrt{L}},\frac{1}{L^{1/4}}$ to show the sharpness of the value $\frac{1}{\sqrt{L}}$. We conduct classification on the MNIST dataset \citep{lecun1998gradient}. We train the model with SGD \footnote{GD exhibits the same phenomenon.  We use SGD due to the  expensive per-iteration cost of GD.} and the size of minibatch is $256$. The learning rate is set to $\eta=0.01$ for all networks without tuning.

We plot the training curves in Figure~\ref{fig:width}. For ResNets with $\tau$, we see that both $\tau =\frac{1}{L}$ and $\tau=\frac{1}{\sqrt{L}}$ are able to train very deep ResNets successfully and $\tau=\frac{1}{\sqrt{L}}$ achieves lower training loss than $\tau=\frac{1}{L}$. For  $\tau=\frac{1}{L^{1/4}}$, the training loss explodes for models with depth $100$ and $1000$. This indicates that the bound $\tau=\frac{1}{\sqrt{L}}$ is sharp for learning deep ResNets. Moreover, the convergence of ResNets with $\tau=\frac{1}{\sqrt{L}}$ does not depend on the depth while training feedforward network becomes harder as the depth increases, corroborating our theory nicely.

To clearly see the benefit of $\tau=\frac{1}{\sqrt{L}}$ over $\tau=\frac{1}{L}$, we conduct the classification task on the CIFAR10 dataset \cite{cifar} with the residual networks from \cite{he2016deep}. {A bit different from the model described in Section \ref{sec:model},  here one residual block is composed of two stacked convolution layers. We argue that our theoretical analysis still applies if treating the number of channels in convolution layer as width in Section \ref{sec:model}.} We plot the training/validation curves in Figure \ref{fig:appcifar}. We can see that with $\tau=\frac{1}{\sqrt{L}}$, both ResNet110 and ResNet1202 can be trained to good accuracy without BN. In contrast, with $\tau=\frac{1}{L}$, the performance of ResNet110 and ResNet1202 drops a lot. 

In the sequel, we use ``adding $\tau^*$" or ``$+\tau^*$" to denote residual network with $\tau=\frac{1}{\sqrt{L}}$. 
\begin{figure}
\centering
    \begin{minipage}{0.5\textwidth}
        \centering    
        \centerline{\includegraphics[width=\columnwidth]{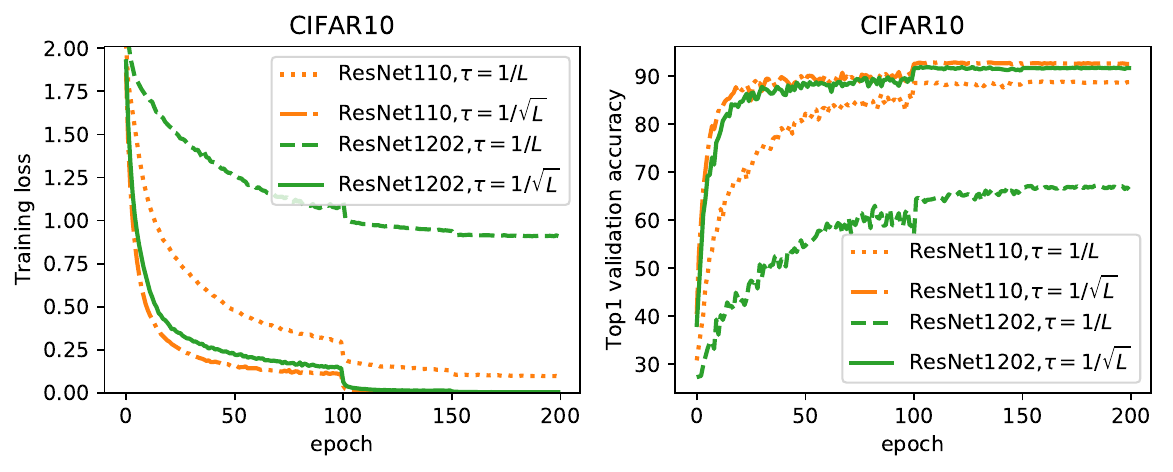}}
        \caption{Training/validation curves of ResNet110/1202 with $\tau=1/\sqrt{L}$ and $\tau=1/L$ for CIFAR10 classification task. We use the models in \cite{he2016deep} and remove all BN layers.}
        \label{fig:appcifar}
    \end{minipage}
    \hfill
    \begin{minipage}{0.45\textwidth}
        \centering
        \begin{subfigure}{.5\textwidth}
          \centering
          \includegraphics[width=1.0\textwidth]{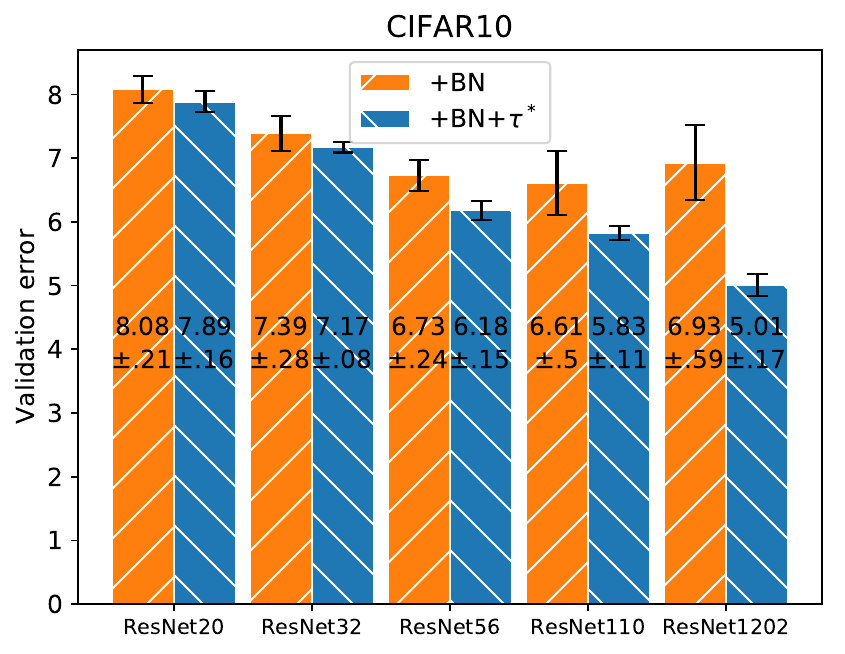}
          \label{fig:sub1}
        \end{subfigure}%
        \begin{subfigure}{.5\textwidth}
          \centering
          \includegraphics[width=1.0\textwidth]{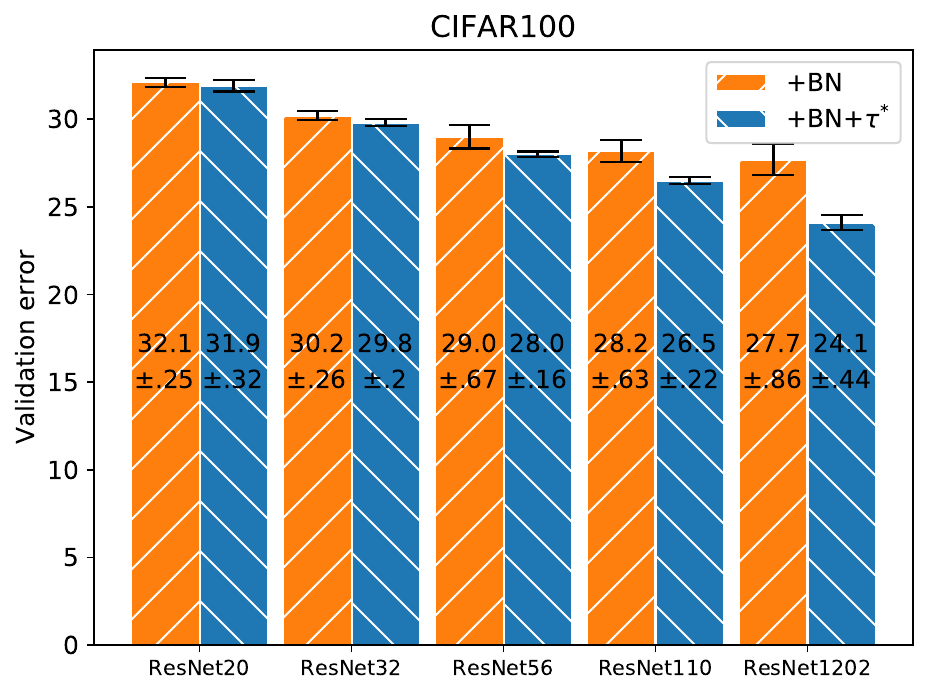}
          \label{fig:sub2}
        \end{subfigure}
        \caption{Validation error bar charts for CIFAR classification tasks. Numbers are average of 5 runs with standard deviations. 
        The deeper network, the larger benefit of $\tau^*$.}
        \label{fig:cifarbn}
    \end{minipage}
    \vspace{-2mm}
\end{figure}

\subsection{Comparison of adding $\tau^*$ and using \emph{Fixup} } \label{subsec:exp:learntau}

In this section we compare our approach of adding $\tau^*$ and the approach of using \emph{Fixup} for training residual networks without BN. We conduct the classification task on the CIFAR10 dataset. We use the residual models in \citep{he2016deep} with removing all the normalization layers.  For the approach of \emph{Fixup}, we use  the code from their github website with the same hyperparameter setting. We note that \emph{Fixup} has a learnable scalar with initial value $1$ on the output of the parametric branch in each residual block, which is equivalent to set $\tau=1$.  For our approach, we use the same model as \emph{Fixup} with setting $\tau=\frac{1}{\sqrt{L}}$ and using the Kaiming initialization instead of \emph{Fixup} initialization. 

The results are presented in Table~\ref{tbl:fixupvstau}. We can see that our approach achieves much  better performance than the \emph{Fixup} approach over all depths. Moreover, the \emph{Fixup} approach fails to converge 2 out of 5 runs for training ResNet1202 and hence the standard deviation is not presented in Table~\ref{tbl:fixupvstau}. 

\begin{table}
\parbox{.45\linewidth}{
\caption{Validation errors of ResNets+Fixup and ResNets+$\tau^*$ on CIFAR10. Numbers are average of 5 runs with standard deviations. }
\label{tbl:fixupvstau}
\begin{center}
\def\arraystretch{1.25}
\begin{tabular}{|c|c|c|}
\hline Model &  + Fixup  &  + $\tau^*$ \\
  \hline ResNet20 & 8.72($\pm$0.26) & \textbf{8.39}($\pm$0.11) \\ \hline
     ResNet32  & 7.99($\pm$0.24) & \textbf{7.68}($\pm$0.10)  \\\hline
   ResNet110 & 7.24($\pm$0.12) & \textbf{6.52}($\pm$0.20) \\\hline
   ResNet1202  & 7.83(N/A) & \textbf{6.08}($\pm$0.21)\\ \hline
\end{tabular}
\end{center}
}
\hfill
\parbox{.45\linewidth}{
\begin{center}
\caption{Top1 validation error on ImageNet.  The models are adapted from \cite{he2016deep}.}
\label{tbl:imgnet}
\vspace{2mm}
\begin{tabular}{ |c|c|c| } 
\hline
Model & Method &  Error \\ 
\hline
\multirow{2}{4em}{ResNet50} & +BN  & 23.6 \\ 
& +BN+$\tau^*$ & \textbf{22.7} \\ 
\hline
\multirow{2}{4em}{ResNet101} & +BN  & 22.0 \\ 
& +BN+$\tau^*$ & \textbf{21.4} \\ 
\hline
\multirow{2}{4em}{ResNet152} & +BN & 21.7 \\ 
& +BN+$\tau^*$ & \textbf{20.9} \\ 
\hline
\end{tabular}
\end{center}
}
\end{table}

\subsection{Add $\tau^*$ on top of normalization}

In this section, we empirically show that adding $\tau^*$ in the residual block with batch normalization can also help to achieve better performance. We conduct experiments on standard classification datasets: CIFAR10/100 and ImageNet. The baseline models are the residual networks in  \cite{he2016deep}. {We note that the residual block here is with batch normalization, which is discussed in Section \ref{subsec:comparison} but not precisely covered by the theoretical model (Section \ref{sec:model}).}  For our approach, the only modification is adding a fixed $\tau=\frac{1}{\sqrt{L}}$ at the output of each residual block (right before the residual addition).  We also tried to use learnable $\tau$ but did not observe gain, which may be due to that the BN layers have learnable scaling factors.  The validation errors  on CIFAR10/100 are illustrated in Figure~\ref{fig:cifarbn},  where all numbers are averaged over five runs. The performance of adding $\tau^*$ is much better than the baseline models and especially the benefit of adding $\tau^*$ becomes larger when the network is deeper. We note that one needs warm-up learning rate to successfully train ResNet1202+BN, while with $\tau^*$ we use the same learning rate schedule for all depths. 

As the models for ImageNet classification has different numbers of residual blocks in each stage, we choose $\tau^*=\frac{1}{\sqrt{L}}$ where $L$ is the average number of blocks over all stages. We take average instead of sum because there exists a BN layer on the output of each stage. All models are trained for $200$ epochs with learning rate divided by $10$ every $60$ epochs. The other hyperparameters are the same as in \cite{he2016deep}. Table~\ref{tbl:imgnet} shows the top 1 validation error results on ImageNet. We can see that just by adding $\tau^*$ on top of BN we can achieve significant performance gain.

\section{Conclusion}\label{sec:conclusion}
In this paper, we provide a non-asymptotic analysis on the  forward/backward stability for ResNet, which unveils that $\tau=1/\sqrt{L}$ is a sharp value in terms of characterizing the stability.  We also bridge  theoretical understanding and practical guide of ResNet structure. We  empirically verify the efficacy of adding $\tau$ for ResNet with/without batch normalization. As the residual block is also widely used in the \emph{Transformer} model \citep{vaswani2017attention}, it is interesting to study the effect of $\tau$ and layer normalization there.

\clearpage

\onecolumn
\appendix

\section{Useful Lemmas}\label{app:useful-lemmas}

First we list several useful bounds on Gaussian distribution.
 
 \begin{lemma}
\label{lem:erfcbound}Suppose $X\sim\cN(0,\sigma^{2})$, then 
\begin{equation}
\begin{aligned}
&\bbP\{|X|\le x\}
\ge 1-\exp\left(-\frac{x^{2}}{2\sigma^{2}}\right),\\
&\bbP\{|X|\le x\}
\le\sqrt{\frac{2}{\pi}}\frac{x}{\sigma}.
\end{aligned}
\end{equation}

\end{lemma}

Another bound is on the spectral norm of random matrix \citep[Corollary 5.35]{vershynin2012introduction}. 
\begin{lemma}\label{lem:spectralnorm}
Let $\bA\in\bbR^{N\times n}$, and entries
of $\bA$ are independent standard Gaussian random variables. Then for every $t\ge0$, with probability at least $1-\exp(-t^{2}/2)$ one has 
\begin{flalign}
s_{\max}(\bA)\le\sqrt{N}+\sqrt{n}+t,
\end{flalign}
where $s_{\max}(\bA)$ are the largest singular value of $\bA$.
\end{lemma}

\section{Spectral Norm Bound at Initialization}\label{app:thm:initial-spectral-norm}

Next we present a spectral norm bound related to the forward process of ResNet with $\tau$. 

\spectralnormtight*

\begin{proof}

Without introducing ambiguity, we drop the superscript $^{(0)}$ for notation simplicity.  We first build the claim for one fixed sample $i\in [n]$ and drop the subscript $i,$ for convenience.  Let $g_l = h_{l-1}+\tau \bW_lh_{l-1}$ and $h_l = \bD_l g_{l}$ for $l=\{a,..., b\}$.  
We will show for a vector $h_{a-1}$ with $\|h_{a-1}\|=1$, we have $\|h_b\|\le 1+c$ with high probability, where
\begin{flalign}
h_b = \bD_b (\bI+\tau\bW_{b})\bD_{b-1}\cdots \bD_{a}(\bI+\tau\bW_{a}) h_{a-1}.
\end{flalign}
Then we have $\|g_l\|\ge \|h_l\|$ due to the assumption $\|\bD_l\|\le 1 $. Hence we have 
\begin{equation*}
\|h_{b}\|^2 = \frac{\|h_{b}\|^2}{\|h_{b-1}\|^2}  \cdots\frac{\|h_{a}\|^2}{\|h_{a-1}\|^2}\|h_{a-1}\|^2\le \frac{\|g_{b}\|^2}{\|h_{b-1}\|^2}  \cdots\frac{\|g_{a}\|^2}{\|h_{a-1}\|^2}\|h_{a-1}\|^2.
\end{equation*}
Taking logarithm at both side, we have
\begin{equation}
\log{\|h_{b}\|^2}\le \sum_{l=a}^{b}\log \Delta_{l},\quad \quad \text{where  } \Delta_{l} := \frac{\|g_{l}\|^2}{\|h_{l-1}\|^2}. 
\end{equation}
If letting $\tilde{h}_{l-1} := \frac{h_{l-1}}{\|h_{l-1}\|}$, then we obtain that{\small
\begin{equation*}
\begin{aligned}
\log{\Delta_{l}} &= \log\left(1 + 2\tau\left\langle \tilde{h}_{l-1},\bW_{l}\tilde{h}_{l-1} \right\rangle + \tau^{2}\|\bW_{l}\tilde{h}_{l-1}\|^{2}\right)\\
&\leq  2\tau\left\langle \tilde{h}_{l-1},\bW_{l}\tilde{h}_{l-1} \right\rangle + \tau^{2}\|\bW_{l}\tilde{h}_{l-1}\|^{2},
\end{aligned}
\end{equation*}}
where the inequality is due to the fact $\log (1+x) \le x$ for all $x>-1$. Let $\xi_{l} := 2\tau\left\langle \tilde{h}_{l-1},\bW_{l}\tilde{h}_{l-1} \right\rangle$ and $\zeta_{l}:= \tau^{2}\|\bW^{(0)}_{l}\tilde{h}_{l-1}\|^{2}$, then given $h_{l-1}$ we have $\xi_{l}\sim \cN\left(0, \frac{4\tau^2}{m}\right)$, $\zeta_{l}\sim \frac{\tau^{2}}{m}\chi_{m}^2$ because of the random initialization of $\bW_l$.  We see that
\begin{equation}
\begin{aligned}
\mathbb{P}\left(\sum\limits_{l=a}^{b}\log\Delta_{l}\geq c_1\right)&\leq \mathbb{P}\left(\sum\limits_{l=a}^{b}\xi_{l}\geq \frac{c_1}{2}\right) + \mathbb{P}\left(\sum\limits_{l=a}^{b}\zeta_{l}\geq \frac{c_1}{2}\right).
\end{aligned}
\end{equation}
Next we bound  the two terms on the right hand side one by one. For the first term we have
\begin{equation}
\begin{aligned}
\mathbb{P}\left(\sum\limits_{l=a}^{b}\xi_{l}\geq \frac{c_1}{2}\right)=\mathbb{P}\left(\exp\left(\lambda\sum\limits_{l=a}^{b}\xi_{l}\right)\geq \exp\left(\frac{\lambda c_1}{2}\right)\right) \leq \mathbb{E}\left[\exp\left(\lambda\sum\limits_{l=a}^{b}\xi_{l} - \frac{\lambda c_1}{2}\right)\right],
\end{aligned}
\end{equation}
where $\lambda$ is any positive number and the last inequality uses the Markov's inequality. Moreover,
\begin{equation}
\begin{aligned}
\mathbb{E}\left[\exp\left(\lambda\sum\limits_{l=a}^{b}\xi_{l}\right)\right] &= \mathbb{E}\left[\exp\left(\lambda\sum\limits_{l=a}^{b-1}\xi_{l}\right)\mathbb{E}\left[\exp\left(\lambda\xi_{b}\right)\right]\Big| \mathcal{F}_{b-1}\right]\\
&= \exp\left(\frac{4\tau^{2}\lambda^{2}}{m}\right)\mathbb{E}\left[\exp\left(\lambda\sum\limits_{l=a}^{b-1}\xi_{l}\right)\right]\\
&=\cdots=\exp\left(\frac{4\tau^{2}\lambda^{2}(b - a + 1)}{m}\right).
\end{aligned}
\end{equation} 
Hence we obtain
\begin{equation}
\mathbb{P}\left(\sum\limits_{l=a}^{b}\xi_{l}\geq \frac{c_1}{2}\right) \leq \exp\left(\frac{4m^2c_1^2\tau^{2}(b - a + 1)}{256m\tau^4L^2} -\frac{mc_1^{2}}{32\tau^{2}L}\right) = \exp\left(-\frac{mc_1^{2}}{64\tau^{2}L}\right),
\label{eq:gaussian-tail-bound}
\end{equation}
by choosing $\lambda = \frac{mc_1}{16\tau^{2}L}$ and using $b-a+1 \le L$. Due to the symmetry of $\sum_{l=a}^{b}\xi_{l}$, the conclusion can be generalized to the quantity $|\sum_{l=a}^{b}\xi_{l}|$  that $\mathbb{P}\left(\left|\sum\limits_{l=a}^{b}\xi_{l}\right|\geq \frac{c_1}{2}\right) \le 2\exp\left(-\frac{mc_1^{2}}{64\tau^{2}L}\right)$.

Then, for the second term, we follow the above procedure but for a $\chi_m^2$ variable. We note that the generate moment function of $\chi_{m}^{2}$ is $(1-2t)^{-m/2}$ for $t<1/2$. We will use an inequality that $(1-\frac{x}{m})^{-m}\le e^{x}$ for $x\ge 0$. 
By using the Markov's inequality, we first have for any $\lambda>0$,
\begin{equation}
\begin{aligned}
\mathbb{P}\left(\sum\limits_{l=a}^{b}\zeta_{l}\geq \frac{c_1}{2}\right)=\mathbb{P}\left(\exp\left(\lambda\sum\limits_{l=a}^{b}\zeta_{l}\right)\geq \exp\left(\frac{\lambda c_1}{2}\right)\right) \leq \mathbb{E}\left[\exp\left(\lambda\sum\limits_{l=a}^{b}\zeta_{l} - \frac{\lambda c_1}{2}\right)\right].
\end{aligned}
\end{equation}
Then we have 
\begin{equation}
\begin{aligned}
\mathbb{E}\left[\exp\left(\lambda\sum\limits_{l=a}^{b}\zeta_{l}\right)\right] &= \mathbb{E}\left[\exp\left(\lambda\sum\limits_{l=a}^{b-1}\zeta_{l}\right)\mathbb{E}\left[\exp\left(\lambda\zeta_{b}\right)\right]\Big| \mathcal{F}_{b-1}\right]\\
&=\left (1-\frac{\lambda \tau^2}{m/2}\right)^{-m/2}\mathbb{E}\left[\exp\left(\lambda\sum\limits_{l=a}^{b-1}\zeta_{l}\right)\right]\\
& \le \exp(\lambda\tau^2) \mathbb{E}\left[\exp\left(\lambda\sum\limits_{l=a}^{b-1}\zeta_{l}\right)\right]\\
&\le\cdots\le\exp\left(\lambda\tau^2(b - a + 1)\right).
\end{aligned}
\end{equation} 
Hence we obtain 
\begin{equation}
\begin{aligned}
\mathbb{P}\left(\sum\limits_{l=a}^{b}\zeta_{l}\geq \frac{c_1}{2}\right) \leq \exp\left(\lambda\tau^2(b - a + 1)- \frac{\lambda c_1}{2}\right) \le \exp\left(-\frac{mc_1^2}{2\tau^2L}\left(1-\frac{2\tau^2L}{c_1}\right)\right),
\label{eq:normal}
\end{aligned}
\end{equation}
by choosing $\lambda = \frac{mc_1}{\tau^2L}$ and using  $b-a+1 \le L$. If further setting $\tau$ such that $\tau^2 L\le \frac{c_1}{4}$, we have 
\begin{flalign}
\mathbb{P}\left(\sum\limits_{l=a}^{b}\zeta_{l}\geq \frac{c_1}{2}\right) \le \exp\left(-\frac{mc_1^2}{4\tau^2L}\right). \label{eq:chi-tail-bound}
\end{flalign}

Combining \eqref{eq:gaussian-tail-bound} and \eqref{eq:chi-tail-bound},  we obtain $\mathbb{P}\left(\sum\limits_{l=a}^{b}\log\Delta_{l}\geq c_1\right)\le 3\exp\left(-\frac{mc_1^2}{64\tau^2L}\right)$ under the condition $\tau^2 L\leq \frac{c_1}{4}$. Hence we have 
$\mathbb{P}\left(\|h_b\|\geq 1+c\right) \le \mathbb{P}\left(\sum\limits_{l=a}^{b}\log\Delta_{l}\geq 2\log(1+c)\right)\le 3\exp\left(-\frac{m\log^2(1+c)}{16\tau^2L}\right)$ under the condition that $\tau^2 L\leq  \frac{1}{2} \log (1+c)$. We next use $\epsilon$-net argument to prove the claim
for all $m$-dimensional vectors of $h_{a-1}$. Let $\cN_{\epsilon}$ be an $\varepsilon$-net over the unit ball in $\bbR^m$ with $\epsilon < 1$, then we have the cardinality $|\cN_{\epsilon}|\le (1+2/\epsilon)^m $. 
Taking the union bound over all vectors $h_{a-1}$ in the net $\cN_{\epsilon}$, we obtain 
\begin{flalign*}
 \bbP\left\{\max_{h_{a-1}\in \cN_{\epsilon}}\|h_b\|> 1+c \right\}&\le (1+2/\epsilon)^m \cdot 3\exp\left(-\frac{m\log^2(1+c)}{16\tau^2L}\right)\\ & 
= 3\exp\left(-m \left(\frac{\log^2(1+c)}{16\tau^2L} - \log (1+2/\epsilon)\right)\right) \le 3\exp\left(-m\right),
\end{flalign*}
where the last equality is obtained by choosing $\tau$ appropriately to make $\frac{\log^2 (1+c)}{16\tau^2L} -\log (1+2/\epsilon)>1$. Then we have the spectral norm bound 
\begin{flalign*}
 \left\|\bD_b\left(\bI+\tau\bW_{b}^{(0)}\right)\bD_{b-1}\cdots \bD_{a}\left(\bI+\tau\bW_{a}^{(0)}\right)\right\| \le (1-\epsilon)^{-1} \max_{h_{a-1}\in \cN_\epsilon} \|h_b\|.
\end{flalign*}
This is because of the following argument. For a matrix $\bM$, $v_i$ is a vector in the net which is closest to a unit vector $v$, then $\|\bM v\|\le \|\bM v_i \| + \|\bM (v-v_i)\|\le \|\bM v_i \| + \epsilon \|\bM\|$, 
and hence taking the supremum over $v$, one obtains
$(1-\epsilon) \|\bM \| \le \max_i \|\bM v_i\|$.

Finally taking a union bound over $a$ and $b$ with $1\le a\le b <L$ and a union bound over all samples $i\in [n]$,  we have the claimed result.

\end{proof}

\section{Bounded Forward/Backward Process}\label{app:sec:boundedforward}

\subsection{Proof at Initialization} \label{app:thm:hnorm-initialization}

\hnorminitialization*

\begin{proof}
We ignore the subscript $^{(0)}$ for simplicity. First we have
\begin{equation}
\|h_{i,l}\| = \|h_{i,0}\|\frac{\|h_{i,1}\|}{\|h_{i,0}\|}\cdots\frac{\|h_{i,l}\|}{\|h_{i,l-1}\|}.
\end{equation}
Then we see
\begin{equation}
\begin{aligned}
	\log{\|h_{i,l}\|^{2}} &= \log{\|h_{i,0}\|^{2}} +  \sum\limits_{a = 1}^{l}\log{\frac{\|h_{i, a}\|^{2}}{\|h_{i,a-1}\|^{2}}}= \log{\|h_{i,0}\|^{2}} + \sum\limits_{a = 1}^{l}\log\left(1 + \frac{\|h_{i, a}\|^{2} - \|h_{i,a-1}\|^{2}}{\|h_{i,a-1}\|^{2}}\right). \label{eq:hlnorm1}
\end{aligned}
\end{equation}
We introduce notation $\Delta_{a}:=\frac{\|h_{i,a}\|^{2} - \|h_{i,a-1}\|^{2}}{\|h_{i,a-1}\|^{2}}$. We next give a lower bound on $\Delta_{a}$. Let $S$ be the set $\{k: k\in[m] \text{ and } (h_{i,a-1})_{k} +\tau (\bW_{a}h_{i,a-1})_{k} >0\}$. We have that
\begin{align}
\Delta_{a} &= \frac{1}{\|h_{i,a-1}\|^{2}}\sum\limits_{k\in S}\left[(h_{i,a-1})_{k}^{2} + 2\tau (h_{i,a-1})_{k}(\bW_{a}h_{i,a-1})_{k} + (\tau\bW_{a}h_{i,a-1})_{k}^{2}\right] - \frac{1}{\|h_{i,a-1}\|^{2}}\sum\limits_{k=1}^{m}(h_{i,a-1})_{k}^{2} \nn\\
		   &= -\frac{1}{\|h_{i,a-1}\|^{2}}\sum\limits_{k\notin S}(h_{i,a-1})_{k}^{2} + \frac{1}{\|h_{i,a-1}\|^{2}}\sum\limits_{k\in S}\tau^{2}(\bW_{a}h_{i,a-1})_{k}^{2} + \frac{2}{\|h_{i,a-1}\|^{2}}\sum\limits_{k\in S}\tau (h_{i,a-1})_{k}(\bW_{a}h_{i,a-1})_{k}\nn\\
		   &\geq -\frac{1}{\|h_{i,a-1}\|^{2}}\sum\limits_{k=1}^{m}(\tau\bW_{a}h_{i,a-1})^{2} + \frac{2}{\|h_{i,a-1}\|^{2}}\tau\sum\limits_{k=1}^{m}(h_{i,a-1})_{k}(\bW_{a}h_{i,a-1})_{k}\nn\\
		   &= -\frac{\|\tau\bW_{a}h_{i,a-1}\|^{2}}{\|h_{i,a-1}\|^{2}} + \frac{2\tau\left\langle h_{i,a-1},\bW_{a}h_{i,a-1} \right\rangle}{\|h_{i,a-1}\|^{2}}, \label{eq:Delta-lowerbound}
\end{align}
 where the inequality is due to the fact that for $k\notin S$,  $|(h_{i,a-1})_{k}|<|(\tau\bW_{a}h_{i,a-1})_{k}|$ and $(h_{i,a-1})_{k}(\bW_{a}h_{i,a-1})_{k}\le 0$. Let $\xi_{a}:=\frac{2\tau\left\langle h_{i,a-1},\bW_{a}h_{i,a-1} \right\rangle}{\|h_{i,a-1}\|^{2}}$ and $\zeta_{a}:=\frac{\|\tau\bW_{a}h_{i,a-1}\|^{2}}{\|h_{i,a-1}\|^{2}}$, then $\Delta_{a} \ge \xi_{a}- \zeta_{a}$. We note that given $h_{i,a-1}$, $\xi_{a}\sim \cN\left(0, \frac{4\tau^2}{m}\right)$ and $\zeta_{a}\sim \frac{\tau^{2}}{m}\chi_{m}^2$.  We use a tail bound for a $\chi^2_m$ variable $X$ (see Lemma 1 in \cite{laurent2000adaptive})
\begin{equation}
\mathbb{P}\left(|X-m|\geq u\right)\leq e^{-\frac{u^{2}}{4m}}.
\end{equation}
By applying the tail bound on Gaussian and Chi-square variables, for a constant $c_0$ such that $4\tau^2\le c_0 $ we have 
\begin{equation}
\begin{aligned}
\mathbb{P}\left(\Delta_{a} < -c_0\right) &= \mathbb{P}\left(\Delta_{a} < -c_0\text{ and } \xi_{a} < -\frac{c_0}{2}\right) +  \mathbb{P}\left(\Delta_{a} < -c_0\text{ and } \xi_{a} \ge -\frac{c_0}{2}\right)\\
&\leq \mathbb{P}\left(\xi_{a} < -\frac{c_0}{2}\right) + \mathbb{P}\left(\zeta_{a}>\frac{c_0}{2}\right)\\
& =\frac{1}{2}\exp\left(-\frac{mc_0^2}{32\tau^2}\right) + \exp\left(-\frac{mc_0^2}{16\tau^4}\right)\\
&<\exp\left(-\frac{mc_0^2}{32\tau^2}\right).
\end{aligned}
\end{equation}

Thus, by choosing $c_0 = 0.5$, we have $\mathbb{P}\left(\Delta_a \ge -0.5, \forall a\in [L-1]\right) \ge 1- L\exp\left(-\frac{m}{128\tau^2}\right)$. On the event $\{\Delta_a \ge -0.5, \forall a \in [L-1]\}$, we can use the relation $\log(1 + x)\geq x - x^{2}$ for $x\ge-0.5$ and have
 \begin{equation}
\begin{aligned}
\eqref{eq:hlnorm1}\geq \log{\|h_{i,0}\|^{2}} + \sum\limits_{a = 1}^{l}\left(\Delta_{a} - \Delta_{a}^{2}\right).
\end{aligned}
\end{equation}

 Due to \eqref{eq:gaussian-tail-bound} and \eqref{eq:chi-tail-bound}, we have for any $c_1>0$, and $\tau^2L\leq c_1/4$,
\begin{equation}\label{eq:sum}
\begin{aligned}
&\mathbb{P}\left(\sum\limits_{l=a}^{b}\xi_{l}\geq \frac{c_1}{2}\right)\leq \exp\left(-\frac{mc_1^{2}}{64\tau^{2}L}\right),\; \mathbb{P}\left(\sum\limits_{l=a}^{b}\xi_{l}<- \frac{c_1}{2}\right)\leq \exp\left(-\frac{mc_1^{2}}{64\tau^{2}L}\right),\\
&\mathbb{P}\left(\sum\limits_{l=a}^{b}\zeta_{l}\geq \frac{c_1}{2}\right) \leq \exp\left(-\frac{mc_1^2}{4\tau^2L}\right).
\end{aligned}
\end{equation}
Thus we have for any $c_1>0$, and $\tau^2L\leq c_1/4$,
 \begin{equation}\label{eq:prob diff}
\begin{aligned}
\mathbb{P}\left(\sum_{a=1}^{l}\Delta_{a}\leq -c_1\right) &= \mathbb{P}\left(\sum_{a=1}^{l}\Delta_{a}\leq -c_1, \sum_{a=1}^{l}\xi_{a}\geq -\frac{c_1}{2}\right) + \mathbb{P}\left(\sum_{a=1}^{l}\Delta_{a}\leq -c_1, \sum_{a=1}^{l}\xi_{a}\leq -\frac{c_1}{2}\right)\\
&\leq \mathbb{P}\left( \sum_{a=1}^{l}\zeta_{a}\geq \frac{c_1}{2}\right) + \mathbb{P}\left(\sum_{a=1}^{l}\xi_{a}\leq -\frac{c_1}{2}\right) =2\exp\left(-\frac{mc_1^{2}}{64\tau^{2}L}\right).
\end{aligned}
\end{equation}
We can derive a similar result that $\mathbb{P}\left(\sum_{a=1}^{l}\Delta_{a}\geq c_1\right)\leq \mathbb{P}\left(\sum_{a=1}^{l}\xi_{a}\ge c_1\right)  \le \exp\left(-\frac{mc_1^{2}}{16\tau^{2}L}\right)$.
Let $a = b$ in \eqref{eq:sum}, we have obtained that for a single $\Delta_{a}$, for a constant $c_1$ such that $4\tau^2\le c_1 $,
\begin{equation}
    \mathbb{P}\left(\left|\Delta_{a}\right|\geq c_1\right)\leq 2\exp\left(-\frac{mc_1^{2}}{32\tau^2}\right).
\end{equation}
In addition, we see that for any $16\tau^4L\le c_1 $
\begin{equation}
    \mathbb{P}\left(\sum\limits_{a=1}^{l}\Delta^{2}_{a}\geq c_1\right)\leq \sum\limits_{a=1}^{l}\mathbb{P}\left(\Delta_{a}^{2}\geq \frac{c_1}{l}\right) = \sum\limits_{a=1}^{l}\mathbb{P}\left(|\Delta_{a}|\geq \sqrt{\frac{c_1}{l}}\right)\leq 2l \exp\left(-\frac{mc_1}{32}\right).
\end{equation}
Thus, similar to the \eqref{eq:prob diff}, we obtain for any $c_1>0$ and $8\tau^2L< c_1$,
\begin{equation}
\begin{aligned}
    \mathbb{P}\left(\sum\limits_{a=1}^{l}\left(\Delta_{a} - \Delta^{2}_{a}\right)\leq -c_1\right) &\leq \mathbb{P}\left(\sum_{a=1}^{l}\Delta_{a}\leq -\frac{c_1}{2}\right) + \mathbb{P}\left(\sum\limits_{a=1}^{l}\Delta^{2}_{a}\geq \frac{c_1}{2}\right) \\
    &\le 2\exp\left(-\frac{mc_1^{2}}{256\tau^{2}L}\right)+ 2(L-1)\exp\left(-\frac{mc_1}{64}\right)\\
    &\le 2L\exp\left(-\frac{mc_1}{64}\right)
\end{aligned}
\end{equation}
Thus  on the event of $\{\Delta_a \ge -0.5, \forall a\in [L-1]\}$, we have for any $c_1>0$ and $8\tau^2L< c_1$,
\begin{equation}
    \mathbb{P}\left(\log{\|h_{i,l}\|^{2}} \leq -c_1 \right)\leq\mathbb{P}\left(\log{\|h_{i,0}\|^{2}} + \sum\limits_{a = 1}^{l}\left(\Delta_{a} - \Delta_{a}^{2}\right)\leq -c_1\right)\leq 2L\exp\left(-\frac{mc_1}{64}\right).
\end{equation}
Then we get the conclusion $\mathbb{P}\left(\|h_{i,l}\|< 1-c\right) = \mathbb{P}\left(\log{\|h_{i,l}\|^{2}} \leq -2\log (1-c)^{-1}\right)\le  2L\exp\left(-\frac{1}{32}m\log(1-c)^{-1}\right)$. Taking union bound over $i\in[n]$ and $l\in[L-1]$, we get the claimed result with probability $1-2nL^2\exp\left(-\frac{1}{32}m\log(1-c)^{-1}\right)$ under the condition $\tau^2L \le \frac{1}{4}\log (1-c)^{-1}$. 
\end{proof}

\subsection{Lemmas and Proofs after Perturbation}\label{app:sec:perturbed-stability}
We use $\overrightarrow{\bW}^{(0)}$ to denote the weight matrices at initialization and use $\overrightarrow{\bW}'$ to denote the perturbation matrices. Let  $\overrightarrow{\bW} = \overrightarrow{\bW}^{(0)} + \overrightarrow{\bW}'$. We define $h_{i,l}^{(0)} = \phi((\bI+\tau\bW_l^{(0)})h_{i,l-1}^{(0)})$ and $h_{i,l} = \phi((\bI+\tau\bW_l)h_{i,l-1})$ for $l\in[L-1]$, and $h_{i,L}^{(0)} = \phi(\bW_L^{(0)}h_{i,L-1}^{(0)})$ and $h_{i,L} = \phi(\bW_L h_{i,L-1})$. Furthermore, let $h'_{i,l} := h_{i,l}- h_{i,l}^{(0)}$ and $\bD'_{i,l} := \bD_{i,l} - \bD_{i,l}^{(0)}$. We note that $\|\cdot\|_0$ is the number of nonzero entries in $\cdot$. In the sequel, we will use notation $O$ and $\Omega$ to simplify the presentation. Then the spectral norm bound after perturbation is as follows. 

\begin{lemma}\label{lem:perturbed-spectral-norm}
Suppose that $\overrightarrow{\bW}^{(0)}$, $\bA$ are randomly generated as in the initialization step, and $\bW'_{1},\dots,\bW'_{L-1}\in\bbR^{m\times m}$ are perturbation matrices with $\|\bW'_l\|<\tau \omega$ for all $l\in[L-1]$ for some $\omega<1$. Suppose $\bD_{i,0},\dots,\bD_{i,L}$  are  diagonal matrices  representing the activation status of sample $i$. If $\tau^2L \le O(1)$, then with probability at least $1-3nL^2\cdot\exp(-\Omega(m))$ over the initialization randomness we have
\begin{flalign}
\|(\bI+\tau\bW_{b}^{(0)}+\tau\bW'_{b})\bD_{i,b-1}\cdots\bD_{i,a}(\bI+\tau\bW_{a}^{(0)}+\tau\bW'_{a})\|\le O(1).
\end{flalign}
\end{lemma}
\begin{proof}
This proof is similar to the proof of Theorem \ref{thm:initial-spectral-norm}.  We first build the claim for one fixed sample $i\in [n]$ and drop the subscript $i,$ for convenience.  
We will show for a vector $h_{a-1}$ with $\|h_{a-1}\|=1$, we have $\|h_b\|\le 1+c$ with high probability, where
\begin{flalign}
h_b = \bD_b (\bI+\tau\bW^{(0)}_{b}+\tau \bW'_{b})\bD_{b-1}\cdots \bD_{a}(\bI+\tau\bW^{(0)}_{a}+\tau \bW'_{a}) h_{a-1}.
\end{flalign}
Let $g_l = h_{l-1}+\tau \bW^{(0)}_lh_{l-1} +\tau \bW'_lh_{l-1}$ and $h_l = \bD_l g_{l}$ for $l=\{a,..., b\}$.   Then we have $\|g_l\|\ge \|h_l\|$ due to the fact $\|\bD_l\|\le 1 $. Hence we have 
\begin{equation*}
\|h_{b}\|^2 = \frac{\|h_{b}\|^2}{\|h_{b-1}\|^2}  \cdots\frac{\|h_{a}\|^2}{\|h_{a-1}\|^2}\|h_{a-1}\|^2\le \frac{\|g_{b}\|^2}{\|h_{b-1}\|^2}  \cdots\frac{\|g_{a}\|^2}{\|h_{a-1}\|^2}\|h_{a-1}\|^2.
\end{equation*}
Taking logarithm at both side, we have
\begin{equation}
\log{\|h_{b}\|^2}\le \sum_{l=a}^{b}\log \Delta_{l},\quad \quad \text{where  } \Delta_{l} := \frac{\|g_{l}\|^2}{\|h_{l-1}\|^2}. 
\end{equation}
If letting $\tilde{h}_{l-1} := \frac{h_{l-1}}{\|h_{l-1}\|}$, then we obtain that{\small
\begin{equation*}
\begin{aligned}
\log{\Delta_{l}} &= \log\left(1 + 2\tau\left\langle \tilde{h}_{l-1},\bW^{(0)}_{l}\tilde{h}_{l-1} \right\rangle + \tau^{2}\|\bW^{(0)}_{l}\tilde{h}_{l-1}\|^{2} +  2\tau\left\langle (\bI + \tau\bW^{(0)}_{l})
 \tilde{h}_{l-1}, \bW'_{l}\tilde{h}_{l-1} \right\rangle + \tau^{2}\|\bW'_{l}\tilde{h}_{l-1}\|^{2}\right)\\
&\leq  2\tau\left\langle \tilde{h}_{l-1},\bW^{(0)}_{l}\tilde{h}_{l-1} \right\rangle + \tau^{2}\|\bW^{(0)}_{l}\tilde{h}_{l-1}\|^{2} + 2\tau\left\langle (\bI + \tau\bW^{(0)}_{l})
 \tilde{h}_{l-1}, \bW'_{l}\tilde{h}_{l-1} \right\rangle + \tau^{2}\|\bW'_{l}\tilde{h}_{l-1}\|^{2},
\end{aligned}
\end{equation*}}
where the inequality is due to the fact $\log (1+x) \le x$ for all $x>-1$. 
We can bound the sum over layers of the first two terms as in the proof of Theorem \ref{thm:initial-spectral-norm}. Next we control the last two terms related with $\bW'_l$, on a high probability event $\{\|\bW^{(0)}_l \|\le 4, \text{ for all } l \in [L-1]\}$
\begin{flalign}
&\sum_{l=a}^{b} 2\tau\left\langle (\bI + \tau\bW^{(0)}_{l})
 \tilde{h}_{l-1}, \bW'_{l}\tilde{h}_{l-1} \right\rangle  \le  \sum_{l=a}^{b}2 \tau \|\bI + \tau \bW^{(0)}_l \|\|\bW'_l\| \|\tilde{h}_{l-1}\|^2\le  \sum_{l=a}^{b}2 \tau^2 \omega(1+4\tau),\nn \\
&\sum_{l=a}^{b} \tau^{2}\|\bW'_{l}\tilde{h}_{l-1}\|^{2}  \le  \sum_{l=a}^{b}2 \tau^4 \omega^2.  \nn
\end{flalign}
Hence given $\tau^2 L \le c_1/4$  as in proof of Theorem \ref{thm:initial-spectral-norm} and $\omega$ being a small constant, the above two sum are well controlled. We can obtain a  spectral norm bound as claimed. 
Here the theorem is built for one $\bW'_l$.  At the end of the whole proof, we will see the number of iterations is $\Omega(n^2)$. If we take union bound over all the $\bW'_l$ s running into in the optimization trajectory, the overall probability is still as high as $1 - \Omega(n^3 L^2)\exp(-\Omega(m))$.
\end{proof}

We also have small changes on the output vector of each layer after perturbation. 
\begin{lemma}\label{lem:perturbed-hnorm}
Suppose that $\omega \le O(1)$ and $\tau^2L\le O(1)$. If $\|\bW_{L}'\|\le\omega$ and $\|\bW_{l}'\|\le\tau\omega$ for $l\in[L-1]$, then with probability at least $1-\exp(-\Omega(m\omega^{\nicefrac{2}{3}}))$, the following bounds on $h'_{i,l}$ and $\bD'_{i,l}$ hold for all $i\in[n]$ and all $l\in[L-1]$,

{\small
\begin{flalign}
 & \|h'_{i,l}\|\le O(\tau^2 L\omega),\;\;\|\bD'_{i,l}\|_{0}\le O\left( m(\omega\tau{L})^{\nicefrac{2}{3}}\right),\;\; \|h'_{i,L}\|\le O(\omega),\; \; \|\bD'_{i,L}\|_{0}\le O\left(m\omega^{\nicefrac{2}{3}}\right).\nn
\end{flalign}}
\end{lemma}

\begin{proof}
Fixing $i$ and ignoring the subscript in $i$, by Claim 8.2 in \citet{allen2018convergence}, for $l\in[L-1]$, there exists $\bD''_{l}$ such that $|(\bD''_{l})_{k,k}|\le1$ and  
\begin{flalign}
h'_{l} & =\bD''_{l}\left((\bI+\tau\bW_{l}^{(0)}+\tau\bW'_{l})h_{l-1}-(\bI+\tau\bW_{l}^{(0)})h_{l-1}^{(0)}\right)\nn\\
 & =\bD''_{l}\left((\bI+\tau\bW_{l}^{(0)}+\tau\bW'_{l})h'_{l-1}+\tau\bW'_{l}h_{l-1}^{(0)}\right)\nn\\
 & =\bD''_{l}(\bI+\tau\bW_{l}^{(0)}+\tau\bW'_{l})\bD''_{l-1}(\bI+\tau\bW_{l-1}+\tau\bW'_{l-1})h'_{l-2}\nn\\
 & \quad+\tau\bD''_{l}(\bI+\tau\bW_{l}^{(0)}+\tau\bW'_{l})\bD''_{l-1}\bW'_{l-1}h_{l-2}^{(0)}+\tau\bD''_{l}\bW'_{l}h_{l-1}^{(0)}\nn\\
 & =\cdots\nn\\
 & =\sum_{a=1}^{l}\tau\bD''_{l}(\bI+\tau\bW_{l}^{(0)}+\tau\bW'_{l})\cdots\bD''_{a+1}(\bI+\tau\bW_{a+1}+\tau\bW'_{a+1})\bD''_{a}\bW'_{a}h_{a}^{(0)}.
\end{flalign}
We claim that 
\begin{flalign}
\|h'_{l}\|\le O(\tau^2L\omega)
\end{flalign}
due to the fact $\|\bD''_{l}\|\le1$ and
the assumption $\|\bW'_{l}\|\le\tau\omega$ for $l\in[L-1]$.
This implies that $\|h'_{i,l}\|,\|g'_{i,l}\|\le O(\tau^2L\omega)$ for
all $l\in[L-1]$ and for all $i$ with probability at least $1-O(nL)\cdot\exp(-\Omega(m))$.
One step further, we have $\|h'_{L}\|,\|g'_{L}\|\le O(\omega)$.

As for the sparsity $\|\bD'_{l}\|_{0}$,  we have $\|\bD'_{l}\|_{0}\le O(m(\omega\tau L)^{\nicefrac{2}{3}})$ for every $l=[L-1]$ and $\|\bD'_{L}\|_{0}\le O(m\omega^{\nicefrac{2}{3}})$. 

The argument is as follows (adapt from the Claim 5.3 in \citet{allen2018convergence}).

We first study the case  $l\in[L-1]$. We see that if $(\bD'_{l})_{j,j}\neq0$
one must have $|(g'_{l})_{j}|>|(g_{l}^{(0)})_{j}|$.

We note that $(g_{l}^{(0)})_{j}=(h_{l-1}^{(0)}+\tau\bW_{l}^{(0)}h_{l-1}^{(0)})_{j}\sim\cN\left((h_{l-1}^{(0)})_{j},\frac{\tau^{2}\|h_{l-1}^{(0)}\|^{2}}{m}\right)$.
Let $\xi\le\frac{1}{\sqrt{m}}$be a parameter to be chosen later.
Let $S_{1}\subseteq[m]$ be a index set satisfying $S_{1}:=\{j:|(g_{l}^{(0)})_{j}|\le\xi\tau\}$.
We have $\bbP\{|(g_{l}^{(0)})_{j}|\le\xi\tau\}\le O(\xi\sqrt{m})$ for
each $j\in[m]$. By Chernoff bound, with probability at least $1-\exp(-\Omega(m^{3/2}\xi))$ we have
\begin{flalign*}
|S_{1}|\le O(\xi m^{3/2}).
\end{flalign*}

Let $S_{2}:=\{j:j\notin S_{1},\ \text{and }(\bD'_{l})_{j,j}\neq0\}$.
Then for $j\in S_{2}$, we have $|(g'_{l})_{j}|>\xi\tau$. As we have
proved that $\|g'_{l}\|\le O(\tau^2L\omega)$, we have

\begin{flalign*}
|S_{2}|\le\frac{\|g'_{l}\|^{2}}{(\xi\tau)^{2}}=O((\omega\tau L)^{2}/\xi^{2}).
\end{flalign*}

Choosing $\xi$ to minimize $|S_{1}|+|S_{2}|$, we have $\xi=(\omega\tau L)^{\nicefrac{2}{3}}/\sqrt{m}$
and consequently, $\|\bD'_{l}\|_{0}\le O(m(\omega\tau L)^{\nicefrac{2}{3}})$. Similarly,
we have $\|\bD'_{L}\|_{0}\le O(m\omega^{\nicefrac{2}{3}})$.
\end{proof}

We next prove that the norm of a sparse vector after the ResNet mapping.
\begin{lemma}\label{lem:sparsebound}
Suppose that $s\ge\Omega(d/\log m), \tau^2L\le O(1)$. If $\bW_l$ for $l\in[L]$ satisfy the condition as in Lemma \ref{lem:perturbed-spectral-norm}, then for all $i\in[n]$ and $a\in[L]$ and  for all $s$-sparse vectors $u\in\bbR^{m}$ and for all $v\in\bbR^{d}$, the following bound holds with probability at least $1-(nL)\cdot\exp(-\Omega(s\log m))$
\begin{flalign}
|v^T\bB\bD_{i, L}\bW_{L}\bD_{i, L-1}(\bI+\tau\bW_{L-1})\cdots\bD_{i, a}(\bI+\tau\bW_{a})u|\le O\left(\frac{\sqrt{s\log m}}{\sqrt{d}}\|u\|\|v\|\right),
\end{flalign}
where  $\bD_{i, a}$ is diagonal activation matrix for sample $i$. 
\end{lemma}

\begin{proof}
For any fixed vector $u\in\bbR^{m}$, $\|\bD_{i,L}\bW_{L}\bD_{i,L-1}(\bI+\tau\bW_{L-1})\cdots\bD_{i,a}(\bI+\tau\bW_{a})u\|\le 1.1 \|u\|$ holds with probability at least $1-\exp(-\Omega(m))$ because of Lemma \ref{lem:perturbed-spectral-norm}.

On the above event, for a fixed vector $v\in\bbR^{d}$ and any fixed $\bW_{l}$ for $l\in[L]$,
the randomness only comes from $\bB$, then $v^{T}\bB\bD_{i,L}\bW_{L}\bD_{i,L-1}(\bI+\tau\bW_{L-1})\cdots\bD_{i,a}(\bI+\tau\bW_{a})u$
is a Gaussian variable with mean 0 and variance no larger than $1.1^2\|u\|^2\cdot\|v\|^2/d$.
Hence 
\begin{flalign*}
&\bbP\big\{  |v^{T}\bB\bD_{i,L}\bW_{L}\bD_{i,L-1}(\bI+\tau\bW_{L-1})\cdots\bD_{i,a}(\bI+\tau\bW_{a})u|\ge\sqrt{s\log m}\cdot \Omega(\|u\|\|v\|/\sqrt{d})\big\}\\
 & =\erfc(\Omega(\sqrt{s\log m}))\le\exp(-\Omega(s\log m)).
\end{flalign*}
Take $\epsilon$-net over all $s$-sparse vectors of $u$ and all $d$-dimensional vectors of $v$,  if $s\ge\Omega(d/\log m)$ then with probability $1-\exp(-\Omega(s\log m))$ the claim holds for all $s$-sparse vectors of $u$ and all $d$-dimensional vectors of $v$. Further taking the union bound over all $i\in [n]$ and $a\in [L]$, the lemma is proved.
\end{proof}

\section{Gradient Lower/Upper Bounds and Their Proofs}
Because the gradient is pathological and data-dependent, in order to build bound on the gradient, we need to consider all possible point and all cases of data. Hence we first introduce an arbitrary loss vector and then the gradient bound can be obtained by taking a union bound.

We define the $\bp_{\overrightarrow{\bW}, i}(v, \cdot)$ operator. It back-propagates a vector $v$ to the $\cdot$ which could be the intermediate output $h_l$ or the parameter $\bW_l$  at the specific layer $l$ using the forward propagation state of input $i$ through the network with parameter $ \overrightarrow{\bW}$. Specifically,
\begin{flalign*}
&\bp_{\overrightarrow{\bW}, i}(v, h_l):= (\bI+\tau \bW_{l+1})^T \bD_{i,l+1} \cdots (\bI+\tau\bW_{L-1})^T\bD_{i,L-1}\bW_L^T\bD_{i,L} \bB^T v,\\
&\bp_{\overrightarrow{\bW}, i}(v, \bW_l):=\tau \left(\bD_{i,l} (\bI+\tau \bW_{l+1})^T\cdots (\bI+\tau\bW_{L-1})^T\bD_{i,L-1}\bW_L^T\bD_{i,L} \bB^T v\right) h_{i, l-1}^T \quad  \forall l\in[L-1],\\
&\bp_{\overrightarrow{\bW}, i}(v, \bW_L):=\left(\bD_{i,L} \bB^T v\right) h_{i, L-1}^T.
\end{flalign*}
Moreover, we introduce
\begin{flalign*}
\bp_{\overrightarrow{\bW}}(\overrightarrow{v}, \bW_l):=\sum_{i=1}^n \bp_{\overrightarrow{\bW}, i}(v_i, \bW_l) \quad \forall l\in[L],
\end{flalign*}
where $\overrightarrow{v}$ is composed of $n$ vectors $v_i$ for $i\in [n]$. If $v_i$ is the error signal of input $i$, then $\nabla_{\bW_l} F_i(\overrightarrow{\bW}) = \bp_{\overrightarrow{\bW},i}(\bB h_{i,L}-y_i^* , \bW_l)$.

\subsection{Gradient Upper Bound}\label{app:sec:gradient-upperbound}

\gradientupperbound*

\begin{proof}
We ignore the superscript $^{(0)}$ for simplicity. Then for an $i\in[n]$ we have
\begin{flalign*}
\left\|\nabla_{\bW_{L}}F_i(\overrightarrow{\bW})\right\|_F =\left\|\left(\bD_{i,L}\partial h_{i,L}\right)h_{i,L-1}^T\right\|_F =\left\|\left(\bD_{i,L}\partial h_{i,L}\right)\right\| \left\|h_{i,L-1}^T\right\|\le \frac{1+c}{1-\epsilon} \|\partial h_{i,L}\|,
\end{flalign*}
because of Theorem \ref{thm:initial-spectral-norm}. 
Similarly, we have for $l\in[L-1]$,
\begin{flalign*}
\left\|\nabla_{\bW_{l}}F_i(\overrightarrow{\bW})\right\|_F &=\left\|\tau \left(\bD_{i,l} (\bI+\tau \bW_{l+1})^T\cdots (\bI+\tau\bW_{L-1})^T\bD_{i,L-1}\bW_L^T\bD_{i,L} \partial h_{i,L}\right) h_{i, l-1}^T\right\|_F \\
&\le  \tau\|\bD_{i,l} (\bI+\tau \bW_{l+1})^T\cdots \bD_{i,L-1}\|\cdot \|\bW_L^T\bD_{i,L}\| \cdot \|\partial h_{i,L}\|\cdot \|h_{i,l-1}\|\\ 
&\le \frac{(1+c)^2}{(1-\epsilon)^2}(2\sqrt{2}+c)\tau \|\partial h_{i,L}\|,
\end{flalign*}
because of Theorem \ref{thm:initial-spectral-norm} and Lemma \ref{lem:spectralnorm}. 
\end{proof}

The above upper bounds hold for the initialization $\overrightarrow{\bW}^{(0)}$ because of Theorem~\ref{thm:initial-spectral-norm} and Theorem~\ref{thm:hnorm-initialization}. They also hold for all the $\overrightarrow{\bW}$ such that $\|\overrightarrow{\bW}-\overrightarrow{\bW}^{(0)}\|\le \omega$ due to Lemma~\ref{lem:perturbed-spectral-norm}.

For the quadratic loss function, we have $\|\partial h_{i,L}\|^2= \|\bB^T(\bB h_{i,L}-y_{i}^{*})\|^2= O(m/d)F_i(\overrightarrow{\bW})$. We have the gradient upper bound as follows.
\begin{theorem}\label{app:thm:gradient-upperbound}
Suppose $\omega = O(1)$. For every input sample $i\in[n]$ and for every $l\in[L-1]$ and for every $\overrightarrow{\bW}$ such that $\|\bW_L-\bW_L^{(0)}\|\le \omega$ and $\|\bW_l-\bW_l^{(0)}\|\le \tau\omega$,  the following holds with probability at least $1- O(nL^2)\cdot \exp(-\Omega(m))$ over the randomness of $\bA,\bB$ and $\overrightarrow{\bW}^{(0)}$
\begin{align}\label{eq:gradient-upperbound}
&\|\nabla_{\bW_{l}}F_i(\overrightarrow{\bW})\|_{F}^{2}\le O\left(\frac{\tau^{2}m}{d} F_i(\overrightarrow{\bW})\right), \\ &\|\nabla_{\bW_{L}}F_i(\overrightarrow{\bW})\|_{F}^{2}\le O\left(\frac{m}{d} F_i(\overrightarrow{\bW})\right). 
\end{align}
\end{theorem}

\subsection{Gradient Lower bound}\label{app:thm:gradient-lowerbound}
For the quadratic loss function, we have the following gradient lower bound.
\begin{theorem}\label{thm:gradient-lowerbound}
Let $\omega=O\left(\frac{\delta^{3/2}}{n^{3}\log^{3}m}\right)$.
With probability at least $1-\exp(-\Omega(m\omega^{\nicefrac{2}{3}}))$ over the
randomness of $\overrightarrow{\bW}^{(0)},\bA,\bB$, it satisfies
for every $\overrightarrow{\bW}$
with $\|\overrightarrow{\bW}-\overrightarrow{\bW}^{(0)}\|\le\omega$,
\begin{flalign}
 \|\nabla_{\bW_{L}}F(\overrightarrow{\bW})\|_{F}^{2}\ge\Omega\left(\frac{F(\overrightarrow{\bW})}{dn/\delta}\times m\right).
\end{flalign}
\end{theorem}

This gradient lower bound on $\|\nabla_{\bW_{L}}F(\overrightarrow{\bW})\|_{F}^{2}$ acts like the gradient dominance condition  \citep{zou2019improved, allen2018convergence} except that our range on $\omega$ does not depend on the depth $L$. 

\begin{proof}
The gradient lower-bound at the initialization is given by the Section 6.2 in \citep{allen2018convergence} and the Lemma 4.1 in \citep{zou2019improved} via the smoothed analysis \citep{spielman2004smoothed}: with high probability the gradient is lower-bounded,  although the worst case it might be 0.  We adopt the same proof for the Lemma 4.1 in \cite{zou2019improved} based on two preconditioned results Theorem \ref{thm:hnorm-initialization} and Lemma \ref{lem:separateness}. We shall not repeat it here. 

Now suppose that we have $\|\nabla_{\bW_{L}}F(\overrightarrow{\bW}^{(0)})\|_{F}^{2}\ge\Omega\left(\frac{F(\overrightarrow{\bW}^{(0)})}{dn/\delta}\times m\right)$. We next bound the change of the gradient  after perturbing the parameter. 
Recall that 
\begin{flalign*}
\bp_{\overrightarrow{\bW}^{(0)}}(\overrightarrow{v}, \bW_L)-\bp_{\overrightarrow{\bW}}(\overrightarrow{v}, \bW_L)=\sum_{i=1}^{n}\Big((v_{i}^{T}\bB\bD_{i,L}^{(0)})^{T}(h_{i,L-1}^{(0)})^{T}-(v_{i}^{T}\bB\bD_{i,L})^{T}(h_{i,L-1})^{T}\Big)
\end{flalign*}
By Lemma \ref{lem:perturbed-hnorm} and Lemma \ref{lem:sparsebound}, we know, 
\begin{flalign*}
  \|v_{i}^{T}\bB\bD_{i,L}^{(0)}-v_{i}^{T}\bB\bD_{i,L}\| \le O(\sqrt{m\omega^{\nicefrac{2}{3}}}/\sqrt{d})\cdot\|v_{i}\|.
\end{flalign*}

Furthermore, we know 
\begin{flalign*}
\|v_{i}^{T}\bB\bD_{i,L}\|\le O(\sqrt{m/d})\cdot\|v_{i}\|.
\end{flalign*}
By Theorem \ref{thm:hnorm-initialization} and Lemma \ref{lem:perturbed-hnorm}, we have 
\begin{flalign*}
\|h_{i,L-1}^{(0)}\|\le1.1\quad\text{and }\quad\|h_{i,L-1}-h_{i,L-1}^{(0)}\|\le O(\omega).
\end{flalign*}
Combing the above bounds together, we have 
\begin{flalign*}
\|\bp_{\overrightarrow{\bW}^{(0)}}(\overrightarrow{v}, \bW_L)-\bp_{\overrightarrow{\bW}}(\overrightarrow{v}, \bW_L)\|_{F}^{2}
\le n\|\overrightarrow{v}\|^{2}\cdot O(\sqrt{m\omega^{\nicefrac{2}{3}}/d}+\omega\sqrt{m/d})^{2} \le n\|\overrightarrow{v}\|^{2}\cdot O\left(\frac{m}{d}\omega^{\nicefrac{2}{3}}\right)
\end{flalign*}

Hence the gradient lower bound still holds for $\overrightarrow{\bW}$
given $\omega<O\left(\frac{\delta^{3/2}}{n^{3}}\right)$.

Finally, taking $\epsilon-$net over all possible vectors $\overrightarrow{v}=(v_{1},\dots,v_{n})\in(\bbR^{d})^{n}$,
we prove that the above gradient lower bound holds for all $\overrightarrow{v}$. In particular, we can now plug
in the choice of $v_{i}=\bB h_{i,L}-y_{i}^{*}$ and it implies our
desired bounds on the true gradients. 
\end{proof}

The gradient lower bound requires the following property.
\begin{lemma}\label{lem:separateness} 
For any $\delta$ and any pair $(x_{i},x_{j})$ satisfying $\|x_{i}-x_{j}\|\ge\delta$, then $\|h_{i,l}-h_{j,l}\|\ge\Omega(\delta)$ holds for all $l\in[L]$ with probability at least $1-O(n^{2}L)\cdot\exp(-\Omega(\log^{2}{m}))$ given that $\tau\le O(1/(\sqrt{L}\log{m}))$ and $m\ge \Omega(\tau^2 L^2\delta^{-2})$. 
\end{lemma}
The proof of Lemma \ref{lem:separateness} follows the Appendix C in \cite{allen2018convergence}.

\section{Semi-smoothness for $\tau\le O(1/\sqrt{L})$}\label{app:thm:semismooth}

With the help of Theorem \ref{app:thm:gradient-upperbound} and several other improvements, we can obtain a tighter bound on the semi-smoothness condition of the objective function.

\begin{restatable}{theorem}{semismoothness}\label{thm:semismooth}
Let $\omega = O\left(\frac{\delta^{3/2}}{n^3L^{7/2}}\right)$  and $\tau^2L\le 1$. With high probability, we have for every $\breve{\overrightarrow{\bW}}\in(\bbR^{m\times m})^{L}$ with $\left\|\breve{\overrightarrow{\bW}}-\overrightarrow{\bW}^{(0)}\right\|\le\omega$ and  for every $\overrightarrow{\bW}'\in(\bbR^{m\times m})^{L}$ with $\|\overrightarrow{\bW}'\|\le\omega$, we have 
{\small\begin{flalign*}
F(\breve{\overrightarrow{\bW}}+\overrightarrow{\bW}') \le& F(\breve{\overrightarrow{\bW}})+\langle\nabla F(\breve{\overrightarrow{\bW}}),\overrightarrow{\bW}'\rangle+O(\frac{nm}{d})\|\overrightarrow{\bW}'\|_F^2+O\left(\sqrt{\frac{m}{nd}}\omega^{\frac{1}{3}}L^{\frac{7}{6}}\right)\|\overrightarrow{\bW}'\|_F\sqrt{F(\breve{\overrightarrow{\bW}})}. 
\end{flalign*}}
\end{restatable}

We will show the semi-smoothness theorem for a more general $\omega\in \left[\Omega\left(\left(d/(m\log m)\right)^{\nicefrac{3}{2}}\right),	O(1)\right]$ and the above high probability is at least $1-\exp(-\Omega(m\omega^{\nicefrac{2}{3}}))$
over the randomness of $\overrightarrow{\bW}^{(0)},\bA,\bB$.

Before going to the proof of the theorem, we introduce a lemma. 
\begin{lemma}\label{lem:semilemma}
There exist diagonal matrices $\bD''_{i,l}\in\bbR^{m\times m}$ with
entries in {[}-1,1{]} such that $\forall i\in[n],\forall l\in[L-1]$,
\begin{flalign}
h_{i,l}-\breve{h}_{i,l}=\sum_{a=1}^{l}(\breve{\bD}_{i,l}+\bD''_{i,l})(\bI+\tau\breve{\bW}_{l})\cdots(\bI+\tau\breve{\bW}_{a+1})(\breve{\bD}_{i,a}+\bD''_{i,a})\tau\bW'_{a}h_{i, a-1},
\end{flalign}
and 
\begin{flalign}
h_{i,L}-\breve{h}_{i,L}=&(\breve{\bD}_{i,L}+\bD''_{i,L})\bW'_{L}h_{i, L-1}\nn\\
&+\sum_{a=1}^{L-1}(\breve{\bD}_{i,L}+\bD''_{i,L})\breve{\bW}_{L}\cdots(\bI+\tau\breve{\bW}_{a+1})(\breve{\bD}_{i,a}+\bD''_{i,a})\tau\bW'_{a}h_{i,a-1}.
\end{flalign}

Furthermore, we then have $\forall l\in[L-1],\|h_{i,l}-\breve{h}_{i,l}\|\le O(\tau^2L\omega)$, $\|\bD''_{i,l}\|_{0}\le O(m(\omega\tau L)^{\nicefrac{2}{3}})$, and $\|h_{i,L}-\breve{h}_{i,L}\|\le O((1+\tau\sqrt{L})\|\bW'\|_F)$, $\|\bD''_{i,L}\|_{0}\le O(m\omega^{\nicefrac{2}{3}})$ and 
$$\|\bB h_{i,L}-\bB\breve{h}_{i,L}\|\le O(\sqrt{m/d})\|\bW'\|_F$$
hold with probability $1-\exp(-\Omega(m\omega^{\nicefrac{2}{3}}))$ given $\|\bW'_{L}\|\le \omega, \|\bW'_{l}\|\le \tau \omega$ for $l\in [L-1]$ and  $\omega\le O(1), \tau \sqrt{L}\le 1$. 
\end{lemma}

\begin{proof}[Proof of Theorem \ref{thm:semismooth}] First of all, we know that $\breve{loss}_{i}:=\bB\breve{h}_{i,L}-y_{i}^{*}$
\begin{flalign}
\frac{1}{2}\|\bB h_{i,L}-y_{i}^{*}\|^{2} & =\frac{1}{2}\|\breve{loss}_{i}+\bB(h_{i,L}-\breve{h}_{i,L})\|^{2}\nn\\
 & =\frac{1}{2}\|\breve{loss}_{i}\|^{2}+\breve{loss}_{i}^{T}\bB(h_{i,L}-\breve{h}_{i,L})+\frac{1}{2}\|\bB(h_{i,L}-\breve{h}_{i,L})\|^{2},
\end{flalign}
and 
\begin{flalign}
 & \nabla_{\bW_{l}}F(\overrightarrow{\bW})=\sum_{i=1}^{n}(loss_{i}^{T}\bB\bD_{i,L}\bW_{L}\cdots\bD_{i,l+1}(\bI+\tau\bW_{l+1})\bD_{i,l})^{T}(\tau h_{i,l-1})^{T}.\\
 & \nabla_{\bW_{L}}F(\overrightarrow{\bW})=\sum_{i=1}^{n}(loss_{i}^{T}\bB\bD_{i,L})^{T}(h_{i,l-1})^{T}.
\end{flalign}
We use the relation that for two matrices $A,B$, $\langle A, B\rangle = \text{tr}(A^TB)$. Then, we can write 
\begin{flalign}
\langle\nabla_{\bW_{l}}F(\breve{\overrightarrow{\bW}}),\bW'_{l}\rangle = \sum_{i=1}^{n}(\breve{loss}_{i}^{T}\bB\breve{\bD}_{i,L}\breve{\bW}_{L}\cdots(\bI+\tau\breve{\bW}_{l+1})\breve{\bD}_{i,l}\bW'_{l}(\tau\breve{h}_{i,l-1}).
\end{flalign}
Then further by Lemma 7, we have
\begin{flalign}
 & F(\breve{\overrightarrow{\bW}}+\overrightarrow{\bW}')-F(\breve{\overrightarrow{\bW}})-\langle\nabla F(\breve{\overrightarrow{\bW}}),\overrightarrow{\bW}'\rangle\nn\\
 & =-\langle\nabla F(\breve{\overrightarrow{\bW}}),\overrightarrow{\bW}'\rangle+\frac{1}{2}\sum_{i=1}^{n}\|\bB h_{i,L}-y_{i}^{*}\|^{2}-\|\bB\breve{h}_{i,L}-y_{i}^{*}\|^{2}\nn\\
 & =-\sum_{l=1}^{L}\langle\nabla_{\bW_{l}}F(\breve{\overrightarrow{\bW}}),\bW'_{l}\rangle+\sum_{i=1}^{n}\breve{loss}_{i}^{T}\bB(h_{i,L}-\breve{h}_{i,L})+\frac{1}{2}\|\bB(h_{i,L}-\breve{h}_{i,L})\|^{2}\nn\\
 & \stackrel{(a)}{=}\frac{1}{2}\sum_{i=1}^{n}\|\bB(h_{i,L}-\breve{h}_{i,L})\|^{2}+\sum_{i=1}^{n}\breve{loss}_{i}^{T}\bB\left((\breve{\bD}_{i,L}+\bD''_{i,L})\bW'_{L}h_{i,L-1}-(\breve{\bD}_{i,L})\bW'_{L}\breve{h}_{i,L-1}\right)\nn\\
 & \quad+\sum_{i=1}^{n}\sum_{l=1}^{L-1}\breve{loss}_{i}^{T}\bB\Big((\breve{\bD}_{i,L}+\bD''_{i,L})\breve{\bW}_{L}\cdots(\bI+\tau\breve{\bW}_{l+1})(\breve{\bD}_{i,l}+\bD''_{i,l})\tau\bW'_{l}h_{i,l-1}\nn\\
 & \quad\quad\quad\quad\quad\quad \quad\quad-\breve{\bD}_{i,L}\breve{\bW}_{L}\cdots(\bI+\tau\breve{\bW}_{l+1})\breve{\bD}_{i,l}\bW'_{l}(\tau\breve{h}_{i,l-1})\Big),\label{eq:semiRHS}
\end{flalign}
where (a) is due to Lemma \ref{lem:semilemma}.

We next bound the RHS of \eqref{eq:semiRHS}. We first use Lemma \ref{lem:semilemma}
to get 
\begin{flalign}
\|\bB(h_{i,L}-\breve{h}_{i,L})\|\le O(\sqrt{m/d})\|\bW'\|_F.
\end{flalign}
Next we calculate that for $l=L$, 
\begin{flalign}
 & \Big|\breve{loss}_{i}^{T}\bB\left((\breve{\bD}_{i,L}+\bD''_{i,L})\bW'_{L}h_{i,L-1}-(\breve{\bD}_{i,L})\bW'_{L}\breve{h}_{i,L-1}\right)\Big|\nn\\
 & \le\Big|\breve{loss}_{i}^{T}\bB\left(\bD''_{i,L}\bW'_{L}h_{i,L-1}\right)\Big|+\Big|\breve{loss}_{i}^{T}\bB\left(\breve{\bD}_{i,L}\bW'_{L}(h_{i,L-1}-\breve{h}_{i,L-1})\right)\Big|.
\end{flalign}
For the first term, by Lemma \ref{lem:sparsebound} and Lemma \ref{lem:semilemma}, we have 
\begin{flalign}
\left|\breve{loss}_{i}^{T}\bB\left(\bD''_{i,L}\bW'_{L}h_{i,L-1}\right)\right| &\le O\left(\frac{\sqrt{m\omega^{\nicefrac{2}{3}}}}{\sqrt{d}}\right) \|\breve{loss}_{i}\|\cdot \|\bW'_{L}h_{i,L-1}\|\nn\\
 & \le O\left(\frac{\sqrt{m\omega^{\nicefrac{2}{3}}}}{\sqrt{d}}\right) \|\breve{loss}_{i}\| \cdot\|\bW'_{L}\|,
\end{flalign}
where the last inequality is due to $\|h_{i,L-1}\|\le O(1)$. For the second term, by Lemma \ref{lem:semilemma} we have 
\begin{flalign}
 & \left|\breve{loss}_{i}^{T}\bB\left(\breve{\bD}_{i,L}\bW'_{L}(h_{i,L-1}-\breve{h}_{i,L-1})\right)\right|\nn\\
 & \le\|\breve{loss}_{i}\|\cdot\left\Vert \bB\breve{\bD}_{i,L}\right\Vert _{2}\cdot\|\bW'_{L}\|\|h_{i,L-1}-\breve{h}_{i,L-1}\|\nn\\
 & \le\|\breve{loss}_{i}\|\cdot O\left(\frac{\omega\sqrt{m}}{\sqrt{d}}\right)\cdot \|\bW'_{L}\|,
\end{flalign}
where the last inequality is due to the assumption $\|\bW'_{L}\|\le \omega$. Similarly for $l\in[L-1]$, we ignore the index $i$ for simplicity.
 
\begin{flalign}
 & \Big|\sum_{l=1}^{L-1}\breve{loss}^{T}\Big(\bB(\breve{\bD}_{L}+\bD''_{L})\breve{\bW}_{L}\cdots(\bI+\tau\breve{\bW}_{l+1})(\breve{\bD}_{l}+\bD''_{l})-\bB\breve{\bD}_{L}\breve{\bW}_{L}\cdots(\bI+\tau\breve{\bW}_{l+1})\breve{\bD}_{l}\Big)\bW'_{l}(\tau h_{l-1})\Big|\nn\\
  & =\Big| \sum_{l=1}^{L-1}\breve{loss}^{T}\bB\bD''_{L}\breve{\bW}_{L}(\bD_{L-1}+\bD''_{L-1})(\bI+\tau\breve{\bW}_{L-1})\cdots(\bD_{l}+\bD''_{l})(\tau \bW'_l h_{l-1})\Big|\nn\\
 & \quad +\Big| \sum_{l=1}^{L-1}\sum_{a=l}^{L-1}\breve{loss}^{T}\bB\breve{\bD}_{L}\breve{\bW}_{L}\cdots(\bI+\tau\breve{\bW}_{a+1})\bD''_{a}(\bI+\tau\breve{\bW}_{a})\cdots(\bD_{l}+\bD''_{l})(\tau \bW'_l h_{l-1})\Big|\nn\\
 &\quad + \Big|\sum_{l=1}^{L-1}\breve{loss}^{T}\bB\breve{\bD}_{L}\breve{\bW}_{L}\cdots(\bI+\tau\breve{\bW}_{l+1})\breve{\bD}_{l}\bW'_{l}\tau(h_{l-1}-\breve{h}_{l-1})\Big| \label{eq:last-step}
 \end{flalign}
 
 We next bound the terms in \eqref{eq:last-step} one by one. For the first term, by Lemma \ref{lem:sparsebound} and Lemma \ref{lem:semilemma}, we have
 \begin{flalign}
 &\left|\sum_{l=1}^{L-1}\breve{loss}^{T}\bB\bD''_{L}\breve{\bW}_{L}(\bD_{L-1}+\bD''_{L-1})(\bI+\tau\breve{\bW}_{L-1})\cdots(\bD_{l}+\bD''_{l})(\tau \bW'_l h_{l-1})\right|\nn\\
 &\le O\left(\frac{\sqrt{m\omega^{\nicefrac{2}{3}}}}{\sqrt{d}}\right)\left\|\breve{loss}\right\|\cdot \left\|\sum_{l=1}^{L-1}\breve{\bW}_{L}(\bD_{L-1}+\bD''_{L-1})(\bI+\tau\breve{\bW}_{L-1})\cdots(\bD_{l}+\bD''_{l})(\tau \bW'_l h_{l-1})\right\|\nn\\
 &\stackrel{(a)}{\le} O\left(\frac{\sqrt{m\omega^{\nicefrac{2}{3}}}}{\sqrt{d}}\right) \cdot \|\breve{loss}\|\cdot \tau\sqrt{L}  \|\bW'_{L-1:1}\|_F,
\end{flalign}
where  $\|\bW'_{L-1:1}\|_F=\sqrt{\sum_{l=1}^{L-1}\|\bW'_l\|_F^2}$ and (a) is due to the similar argument \eqref{eq:spectral-to-fro} in the proof Lemma \ref{lem:semilemma} and the fact $\left\|\breve{\bW}_{L}(\bD_{L-1}+\bD''_{L-1})(\bI+\tau\breve{\bW}_{L-1})\cdots(\bD_{l}+\bD''_{l})\right\| = O(1)$ and $\| h_{l-1}\|=O(1)$ holds with high probability. We note that the inequality (a) helps us save a $\sqrt{L}$ factor in our main theorem.

We have similar bound for the second term of \eqref{eq:last-step}
 \begin{flalign}
 &\left|\sum_{l=1}^{L-1}\sum_{a=l}^{L-1}\breve{loss}^{T}\bB\breve{\bD}_{L}\breve{\bW}_{L}\cdots(\bI+\tau\breve{\bW}_{a+1})\bD''_{a}(\bI+\tau\breve{\bW}_{a})\cdots(\bD_{l}+\bD''_{l})(\tau \bW'_l h_{l-1})\right|\nn\\
 &\le O\left(\frac{\sqrt{m(\omega\tau{L})^{\nicefrac{2}{3}}}}{\sqrt{d}}\right) \cdot \|\breve{loss}\|\cdot\tau \sum_{a=1}^{L-1} \sqrt{a}\|\bW'_{a:1}\|_F\nn\\
 &\le O\left(\frac{\sqrt{m(\omega\tau{L})^{\nicefrac{2}{3}}}}{\sqrt{d}}\right) \cdot \|\breve{loss}\|\cdot\tau L^{3/2}\|\bW'_{L-1:1}\|_F.
\end{flalign}

For the last term in \eqref{eq:last-step}, we have 
\begin{flalign}
 & \left|\sum_{l=1}^{L-1}\breve{loss}^{T}\bB\breve{\bD}_{L}\breve{\bW}_{L}\cdots(\bI+\tau\breve{\bW}_{l+1})\breve{\bD}_{l}\bW'_{l}\tau(h_{l-1}-\breve{h}_{l-1})\right|\nn\\
 & \le\|\breve{loss}\|\cdot O\left(\sqrt{m/d}\right)\cdot\sum_{l=1}^{L-1}\|\bW'_l\| \cdot \tau^3L\omega\nn\\
 &\le\|\breve{loss}\|\cdot O\left(\sqrt{m/d}\right)\cdot\|\bW'_{L-1:1}\|_{F}\cdot(\tau^2 L)^{3/2},
\end{flalign}
where is the last inequality is due to the bound on
$\|h_{l-1}-\breve{h}_{l-1}\|$ in Lemma \ref{lem:semilemma}. 
Hence
 \begin{flalign}
 \eqref{eq:last-step} & \le O\left(\frac{\sqrt{m(\omega\tau{L})^{\nicefrac{2}{3}}}}{\sqrt{d}}\right) \cdot \|\breve{loss}\|\cdot\tau L^{3/2}\|\bW_{L-1:1}\|_F\nn\\
  & \le O\left((\tau L)^{\nicefrac{4}{3}}\frac{\sqrt{mL\omega^{\nicefrac{2}{3}}}}{\sqrt{d}}\right) \cdot \|\breve{loss}\|\cdot  \|\bW'_{L-1:1}\|_{F}.
\end{flalign}
Having all the above together and using triangle inequality, we have
the result.
\end{proof}

\begin{proof}[Proof of Lemma \ref{lem:semilemma}]

The proof relies on the following lemma.
\begin{lemma}[Proposition 8.3 in in \citet{allen2018convergence}] \label{prop:8.3}
Given vectors $a, b\in \bbR^m$ and $\bD \in \bbR^{m\times m}$ the diagonal matrix where $\bD_{k,k}=\bone_{a_k\ge 0}$. Then, there exists a diagonal matrix $\bD''\in \bbR^{m\times m}$ with 
\begin{itemize}
\item $|\bD_{k,k}+\bD''_{k,k}|\le 1$ and $|\bD''_{k,k}|\le 1$ for every $k\in [m]$,

\item $\bD''_{k,k}\neq 0 $ only when $\bone_{a_k\ge 0}\neq \bone_{b_k\ge 0}$,

\item $\phi(a)-\phi(b) = (\bD+\bD'')(a-b)$.
\end{itemize}
\end{lemma}

Fixing index $i$ and ignoring the subscript in $i$ for simplicity, by Lemma \ref{prop:8.3}, for each $l\in[L-1]$ there exists a $\bD''_{l}$ such that $|(\bD''_{l})_{k,k}|\le1$ and  
\begin{flalign}
h_l - \breve{h}_{l} & =\phi((\bI+\tau\breve{\bW}_{l}+\tau\bW'_{l})h_{l-1})-\phi((\bI+\tau\breve{\bW}_{l})\breve{h}_{l-1})\nn\\
&=(\breve{\bD}_l+\bD''_{l})\left((\bI+\tau\breve{\bW}_{l}+\tau\bW'_{l})h_{l-1}-(\bI+\tau\breve{\bW}_{l})\breve{h}_{l-1}\right)\nn\\
&=(\breve{\bD}_l+\bD''_{l})(\bI+\tau\breve{\bW}_{l})(h_{l-1} - \breve{h}_{l-1}) + (\breve{\bD}_l+\bD''_{l}) \tau\bW'_{l}h_{l-1}\nn\\
 &= \sum_{a=1}^{l} (\breve{\bD}_{l}+\bD''_{l})(\bI+\tau\breve{\bW}_{l})\cdots(\bI+\tau\breve{\bW}_{a+1})(\breve{\bD}_{a}+\bD''_{a})\tau\bW'_{a}h_{a-1}\nn
\end{flalign}

Then we have following properties. For $l\in[L-1]$, $\|h_l - \breve{h}_{l}\|\le O(\tau^2 L \omega)$. This is because $\|(\breve{\bD}_{l}+\bD''_{l})(\bI+\tau\breve{\bW}_{l})\cdots(\bI+\tau\breve{\bW}_{a+1})(\breve{\bD}_{a}+\bD''_{a})\|\le 1.1$ from Lemma \ref{lem:perturbed-spectral-norm}; $\|h_{a-1}\|\le O(1)$ from Theorem \ref{thm:hnorm-initialization}; and the assumption $\|\bW'_l\|\le \tau \omega$ for $l\in [L-1]$.

To have a tighter bound on $\|h_L-\breve{h}_L\|$, let us introduce $\bW''_{b}:= \sum_{a=b}^{l} (\breve{\bD}_{l}+\bD''_{l})(\bI+\tau\breve{\bW}_{l})\cdots(\bI+\tau\breve{\bW}_{a+1})(\breve{\bD}_{a}+\bD''_{a})\bW'_{a}$, for $b=1, ..., l$. Then we have 
 \begin{flalign}
h_L - \breve{h}_{L}  =\left[\bW''_L, \bW''_{L-1}, ..., \bW''_{1}\right] [ h_{L-1}^T, \tau h_{L-2}^T, ..., \tau h_{0}^T]^T.
\end{flalign}
It is easy to get 
$$ \|[\tau h_{l-1}^T, \tau h_{l-2}^T, ..., \tau h_{0}^T]^T\|= \sqrt{\tau^2 \sum_{a=0}^{l-1}\|h_{a}\|^2}\le \tau \sqrt{L}\cdot O(1),$$
where the inequality is because of  $\|h_{a-1}\|\le O(1)$ from Theorem \ref{thm:hnorm-initialization}.
Next, we have 
\begin{flalign}
\left\|\left[\bW''_l, \bW''_{l-1}, ..., \bW''_{1}\right]\right\|
&= \left\|\left[\bW''_l, \bW''_{l-1}, ..., \bW''_{1}\right]^T\right\| \le \sqrt{\sum_{a=1}^l \|(\bW''_l)^T\|^2}
\le 1.1 \sqrt{\sum_{a=1}^l \|(\bW'_l)^T\|^2}
\le 1.1 \|\bW'_{l:1}\|_F, \label{eq:spectral-to-fro}
\end{flalign}
where the second inequality is from the definition of spectral norm, the third inequality is because of $\|(\breve{\bD}_{l}+\bD''_{l})(\bI+\tau\breve{\bW}_{l})\cdots(\bI+\tau\breve{\bW}_{a+1})(\breve{\bD}_{a}+\bD''_{a})\|\le 1.1$ from Lemma \ref{lem:perturbed-spectral-norm}. 

Hence we have $\|h_L- \breve{h}_{L}\|\le O\left((1+\tau \sqrt{L})\|\bW'\|_F\right) = O\left(\|\bW'\|_F\right)$ because of the assumption $\tau\sqrt{L}\le 1$.

For $l\in[L]$, $\|\bD''_l\|_0\le O(m\omega^{\nicefrac{2}{3}})$. This is because $(\bD''_l)_{k,k}$ is non-zero only at coordinates $k$ where $(\breve{g}_l)_k$ and $(g_l)_k$ have opposite signs, where it holds either $(\bD_l^{(0)})_{k,k}\neq (\breve{\bD}_l)_{k,k}$ or $(\bD_l^{(0)})_{k,k}\neq (\bD_l)_{k,k}$. Therefore by Lemma \ref{lem:perturbed-hnorm}, we have $\|\bD''_l\|_0\le O(m(\omega\tau L)^{\nicefrac{2}{3}})$ if $\|{\bW}'_l\|\le \tau\omega$.

\end{proof}

\section{Proof for Theorem \ref{thm:main-result}}\label{app:thm:main-result}
\mainresult*

\subsection{Convergence Result for GD}

\begin{proof}
Using Theorem \ref{thm:hnorm-initialization} we have $\|h^{(0)}_{i,L}\|\le1.1$ and then using the
randomness of $\bB$, it is easy to show that $\|\bB h_{i,L}^{(0)}-y_{i}^{*}\|^{2}\le O(\log^{2}m)$
with probability at least $1-\exp(-\Omega(\log^{2}m))$, and therefore
\begin{flalign}
F(\overrightarrow{\bW}^{(0)})\le O(n\log^{2}m).
\end{flalign}
Assume that for every $t=0,1,\dots,T-1$, the following holds, 
\begin{flalign}
 & \|\bW_{L}^{(t)}-\bW_{L}^{(0)}\|_{F}\le\omega\stackrel{\Delta}{=}O\left(\frac{\delta^{3/2}}{n^{3}L^{7/2}}\right)\label{eq:neighborassumption}\\
 & \|\bW_{l}^{(t)}-\bW_{l}^{(0)}\|_{F}\le\tau\omega.\label{eq:neighborassumption2}
\end{flalign}
We shall prove the convergence of GD under the assumption \eqref{eq:neighborassumption}
holds, so that previous statements can be applied. At the end, we
shall verify that \eqref{eq:neighborassumption} is indeed satisfied.

Letting $\nabla_{t}=\nabla F(\overrightarrow{\bW}^{(t)})$, we calculate
that 
\begin{flalign}
F(\overrightarrow{\bW}^{(t+1)}) & \le F(\overrightarrow{\bW}^{(t)})-\eta\|\nabla_t\|_{F}^{2}+O(\eta^{2}nm/d)\|\nabla_{t}\|_{F}^{2}+\quad\eta\sqrt{F(\overrightarrow{\bW}^{(t)})}\cdot O\left(\sqrt{\frac{mnL\omega^{\nicefrac{2}{3}}}{d}}(\tau L)^{\nicefrac{4}{3}}\right)\cdot \|\nabla_{t}\|_{F} \nn\\
 & \le\left(1-\Omega\left(\frac{\eta\delta m}{dn}\right)\right)F(\overrightarrow{\bW}^{(t)}),\label{eq:linearconverge}
\end{flalign}
where the first inequality uses Theorem 4, the second inequality uses
the gradient upper bound in Theorem \ref{app:thm:gradient-upperbound} and
the last inequality uses the gradient lower bound in Theorem \ref{thm:gradient-lowerbound}
and the choice of $\eta=O(d/(mn))$ and the assumption on $\omega$ \eqref{eq:neighborassumption}.
That is, after $T=\Omega(\frac{dn}{\eta\delta m})\log\frac{n\log^{2}m}{\epsilon}$
iterations $F(\overrightarrow{\bW}^{(T)})\le\epsilon$.

We need to verify for each $t$, \eqref{eq:neighborassumption} holds. Here we use a result from the Lemma 4.2 in \cite{zou2019improved} that states $\|\bW_L^{(t)}-\bW_L^{(0)}\|_F\le O(\sqrt{\frac{n^2 d\log m }{m\delta}})$.

To guarantee the iterates fall into the region given by $\omega$ \eqref{eq:neighborassumption}, we obtain a bound $m\ge n^8\delta^{-4}dL^7\log^2 m$.
\end{proof}

{
\subsection{Convergence Result for SGD}\label{app:thm:main-result-sgd}

\begin{restatable}{theorem}{mainresultsgd}\label{thm:main-result-sgd}
For the ResNet defined and initialized as in Section \ref{sec:model}, the network width $m\ge \Omega(n^{17}L^7b^{-4}\delta^{-8}d\log^2 m)$. Suppose we do stochastic gradient descent update starting from $\overrightarrow{\bW}^{(0)}$ and 
\begin{flalign}
	\overrightarrow{\bW}^{(t+1)} = \overrightarrow{\bW}^{(t)} - \eta \frac{n}{|S_t|}\sum_{i\in S_t} \nabla F_i(\overrightarrow{\bW}^{(t)}), \label{eq:sgdupdate}
\end{flalign}
where $S_t$ is a random subset of $[n]$ with $|S_t|=b$.  Then with probability at least $1-\exp(-\Omega(\log^{2}m))$, stochastic gradient
descent \eqref{eq:sgdupdate} with learning rate $\eta=\Theta(\frac{db\delta}{n^{3}m\log m})$
finds a point $F(\overrightarrow{\bW})\le\epsilon$ in $T=\Omega(n^{5}b^{-1}\delta^{-2}\log m \log^2 \frac{1}{\epsilon})$ iterations.
\end{restatable}

\begin{proof}
The proof of the case of SGD can be adapted from the proof of Theorem 3.8 in \cite{zou2019improved}.
\end{proof}
}

\section{Proofs of Theorem \ref{thm:converse_tau} and Proposition \ref{clm:norm-with-bn}}\label{app:sec:proof-converse-tau}
\conversetau*

\begin{proof}
By induction we can show for any $k\in[m]$ and $l\in[L-1]$, 
\begin{equation}
(h_{l})_{k}\geq\phi\left(\sum_{a=1}^{l}\left(\tau\bW_{a}h_{a-1}\right)_{k}\right).\label{eq:hl-lowerbound}
\end{equation}
It is easy to verify $(h_1)_k = \phi\left((h_0)_k+(\tau\bW_1h_0)_k\right)\ge \phi\left((\tau\bW_1h_0)_k\right)$ because of $(h_0)_k\ge0$.

Then assume $(h_l)_k \ge \phi\left(\sum_{a=1}^{l}\left(\tau\bW_{a}h_{a-1}\right)_{k}\right)$, we show it holds for $l+1$.{\small
\begin{flalign*}
(h_{l+1})_k&= \phi\left((h_l)_k+(\tau\bW_{l+1}h_l)_k\right)\ge \phi\left(\phi\left(\sum_{a=1}^{l}\left(\tau\bW_{a}h_{a-1}\right)_{k}\right)+(\tau\bW_{l+1}h_l)_k\right)\ge \phi\left(\sum_{a=1}^{l+1}\left(\tau\bW_{a}h_{a-1}\right)_{k}\right),
\end{flalign*}}
where the last inequality can be shown by case study.

Next we can compute the mean and variance of $\sum_{a=1}^{l}\left(\tau\bW_{a}h_{a-1}\right)_{k}$ by taking iterative conditioning. We have  
\begin{flalign}
\mathbb{E} \sum_{a=1}^{l}\left(\tau\bW_{a}h_{a-1}\right)_{k}=0,\quad \mathbb{E} \left(\sum_{a=1}^{l}\left(\tau\bW_{a}h_{a-1}\right)_{k}\right)^2=\frac{\tau^2}{m}\sum_{a=1}^l\mathbb{E}\|h_{a-1}\|^2. \label{eq:variance}
\end{flalign}

Moreover, $(\tau \bW_a h_{a-1})_{k}$ are jointly Gaussian for all $a$ with mean $0$ because $\bW_a$'s are drawn from independent Gaussian distributions. 
We use $l=2$ as an example to illustrate the conclusion, it can be generalized to other $l$. Assume that $h_{0}$ is fixed. First it is easy to verify that $(\tau \bW_1 h_{0})_{k}$  is Gaussian variable with mean $0$ and $(\tau \bW_2 h_{1})_{k}\big|\bW_1$  is also Gaussian variable with mean $0$. Hence $[(\tau \bW_1 h_{0})_{k}, (\tau \bW_2 h_{1})_{k}]$ follows jointly Gaussian with mean vector $[0,0]$. Thus $(\tau \bW_1 h_{0})_{k}+(\tau \bW_2 h_{1})_{k}$ is Gaussian with mean $0$. By induction, we have $\sum_{a=1}^l(\tau \bW_a h_{a-1})_{k}$ is Gaussian with mean $0$. Then we have 
\begin{flalign}
\mathbb{E} \|h_{l}\|^2&\ge \sum_{k=1}^{m}\mathbb{E}\left(\phi\left(\sum_{a=1}^{l}\left(\tau\bW_{a}h_{a-1}\right)_{k}\right)\right)^{2} =\sum_{k=1}^{m}\frac{1}{2}\mathbb{E}\left(\sum_{a=1}^{l}\left(\tau\bW_{a}h_{a-1}\right)_{k}\right)^{2} \nn\\
&=\frac{1}{2} \sum_{k=1}^{m}\frac{\tau^{2}\sum_{a=1}^{l}\mathbb{E}\left[\|h_{a-1}\|^{2}\right]}{m} = \frac{\tau^{2}}{2}\sum_{a=1}^{l}\mathbb{E}\|h_{a-1}\|^{2},\label{eq:finalstep}
\end{flalign}
where the first step is due to \eqref{eq:hl-lowerbound}, the second step is due to the symmetry of Gaussian distribution and the third step is due to \eqref{eq:variance}. Since $(h_{l})_{k} = \phi\left((h_{l-1})_{k} + \left(\bW_{l}h_{l-1}\right)_{k}\right)$, we can show $\mathbb{E}(h_l)_k^2 \ge (h_{l-1})_{k}^2$ given $h_{l-1}$ by numerical integral of Gaussian variable over an interval. Hence we have $\mathbb{E}\|h_{l}\|^{2}\ge\mathbb{E} \|h_{l-1}\|^2\ge\cdots\ge \mathbb{E} \|h_{0}\|^2=1$ by iteratively taking conditional expectation. Then combined with \eqref{eq:finalstep} and the choice of $\tau=L^{-\frac{1}{2}+c}$, we have $\mathbb{E}\|h_{L-1}\|^2 \ge \frac{1}{2}L^{2c}$. Because $(\bW_L)_{i,j}\sim \cN(0,2/m)$ and $h_L=\phi(\bW_L h_{L-1})$, we have $\mathbb{E} \|h_L\|^2 = \|h_{L-1}\|^2$. Thus, the claim is proved.
\end{proof}

\normwithbn*

\begin{proof}
From the inequality (\ref{eq:hl-lowerbound}) in the previous proof, we know for any $k\in[m]$ and $l\in[L-1]$, 
\begin{equation}
(h_{l})_{k}\geq\phi\left(\sum_{a=1}^{l}\left(\tilde{z}_{a}\right)_{k}\right).\label{eq:hl-lowerbound-bn}
\end{equation}

Next we can compute the mean and variance of $\sum_{a=1}^{l}\left(\tilde{z}_a\right)_{k}$ by taking iterative conditioning. We have  
\begin{flalign}
\mathbb{E} \sum_{a=1}^{l}\left(\tilde{z}_a\right)_{k}=0,\quad \mathbb{E} \left(\sum_{a=1}^{l}\left(\tilde{z}_a\right)_{k}\right)^2=\sum_{a=1}^l\mathbb{E}((\tilde{z}_a)_k)^2= l. \label{eq:variance}
\end{flalign}

Then we have 
\begin{flalign}
\mathbb{E} \|h_{l}\|^2\ge \sum_{k=1}^{m}\mathbb{E}\left(\phi\left(\sum_{a=1}^{l}\left(\tilde{z}_a\right)_{k}\right)\right)^{2} = \frac{1}{2}\sum_{k=1}^{m}\mathbb{E}\left[\sum_{a=1}^{l}\left(\tilde{z}_a\right)_{k}\right]^2 = \frac{1}{2}ml,\label{eq:finalstep-bn}
\end{flalign}
where the first step is due to \eqref{eq:hl-lowerbound}, the second step is due to the symmetry of random variable $(\tilde{z}_a)_k$ and the third step is due to \eqref{eq:variance}. The proposition is proved.
\end{proof}

\section{More Empirical Studies}
\label{app:more_experiment}

{
\begin{figure}
\centering
  \includegraphics[width=0.95\linewidth]{figs/journal_fig4.png}
  \caption{Validation accuracy on CIFAR10 of ResNets with different choices of $\tau$ ($\tau=1/L$, $\tau=1/\sqrt{L}$, $\tau=1/L^{1/4}$).  }
  \label{fig:resnet_trn_curve}
\end{figure}

\begin{table}
\caption{Validation accuracy of ResNet110+$\tau$ with different learning rates. }
\label{tbl:vary_lr}
\def\arraystretch{1.25}
\begin{center}
\begin{tabular}{|c|c|c|}
\hline Lr &  $\tau=1/L$  &  $\tau=1/\sqrt{L}$ \\
  \hline 0.1 & 82.7 & 92.2 \\ \hline
     0.2  & 85.6 & \textbf{92.5}  \\\hline
  0.4  & \textbf{86.8} & 92.2  \\\hline
   0.8 & 86.3 & 90.7 \\\hline
   1.6  & 84.4 & 10.0 \\ \hline
\end{tabular}
\end{center}
\end{table}

We do more experiments to demonstrate the points in Section \ref{sec:experiment}.

Besides the basic feedforward structure in Section \ref{subsec:experiment-theory},   we do another experiment to demonstrate that $\tau=1/\sqrt{L}$ is sharp with practical structures (see Figure \ref{fig:resnet_trn_curve}). We can see that for ResNet110 and ResNet1202, $\tau=1/L^{1/4}$ cannot train the network effectively. 

One may wonder if we can tune the learning rate for the case of $\tau=1/L$ to achieve  validation accuracy as well as the case of $\tau=1/\sqrt{L}$. We do a new experiment to verify this (see Table \ref{tbl:vary_lr}). Specifically, for ResNet110 with fixed $\tau=1/L$ and $\tau=1/\sqrt{L}$ on CIFAR10 classification task, we tune the learning rate from 0.1 to 1.6 and record the validation accuracy in Table \ref{tbl:vary_lr}. We can see that ResNet110 with $\tau=1/L$  performs inferior to that with $\tau=1/\sqrt{L}$ even with grid search of learning rates. It is possible that we can achieve a bit better  performance by adjusting the learning rate for $\tau=1/L$. But this requires tuning for each depth.  In contrast, we have shown that with $\tau=1/\sqrt{L}$, one learning rate fits for all nets with different depths.
}

\end{document}